\documentclass[final]{cvpr}

\usepackage{times}
\usepackage{epsfig}
\usepackage{graphicx}
\usepackage{amsmath}
\usepackage{amssymb}
\usepackage{amsthm}
\usepackage{color}
\usepackage{wrapfig,lipsum,booktabs}
\usepackage[ruled,vlined]{algorithm2e}
\usepackage[dvipsnames]{xcolor}

\usepackage{pifont} % http://ctan.org/pkg/pifont
\usepackage{enumitem}
\newtheorem{theorem}{Theorem}
\newcommand{\cmark}{\ding{51}}%
\newcommand{\xmark}{\ding{55}}%

\newcommand{\mbb}[1]{\mathbb{#1}}
\newcommand{\mrm}[1]{\mathrm{#1}}
\newcommand{\mcal}[1]{\mathcal{#1}}

\newcommand{\p}{\partial}

\usepackage[pagebackref=true,breaklinks=true,colorlinks,bookmarks=false]{hyperref}

\def\cvprPaperID{4909} % *** Enter the CVPR Paper ID here
\def\confYear{CVPR 2021}

\begin{document}

%%%%%%%%% TITLE
\title{Learning by Planning: Language-Guided Global Image Editing}

\author{Jing Shi$^{1}$\quad Ning Xu$^2$\quad Yihang Xu$^1$\quad Trung Bui$^2$\quad  Franck Dernoncourt$^2$\quad Chenliang Xu$^1$\\
$^1$University of Rochester \quad\quad\quad\quad $^2$Adobe Research\\
$^1${\tt\small \{j.shi,chenliang.xu\}@rochester.edu} \quad 
$^1${\tt\small yxu74@u.rochester.edu} \quad 
$^2${\tt\small \{nxu,bui,dernonco\}@adobe.com}}

\maketitle
\pagestyle{empty}
\thispagestyle{empty}

%%%%%%%%% ABSTRACT
\begin{abstract}
Recently, language-guided global image editing draws increasing attention with growing application potentials.
However, previous GAN-based methods are not only confined to domain-specific, low-resolution data but also lacking in interpretability.
To overcome the collective difficulties, we develop a \emph{text-to-operation} model to map the vague editing language request into a series of editing operations, e.g., change contrast, brightness, and saturation. Each operation is interpretable and differentiable.
Furthermore, the only supervision in the task is the target image, which is insufficient for a stable training of sequential decisions. Hence, we propose a novel operation planning algorithm to generate possible editing sequences from the target image as pseudo ground truth.
Comparison experiments on the newly collected MA5k-Req dataset and GIER dataset show the advantages of our methods. Code is available at \url{https://jshi31.github.io/T2ONet}.
\end{abstract}
\vspace{-3mm}

%%%%%%%%% BODY TEXT
\section{Introduction}
\label{sec:intro}
%% 1. BACKGROUND: Briefly describe the background. Motivate this work.
% introduce the importance of LDIE: good for novice, convinient in mobile device.
Image editing is ubiquitous in our daily life, especially when posting photos on social media such as Instagram or Facebook. 
However, editing images using professional software like PhotoShop requires background knowledge for image processing and is time-consuming for the novices who want to quickly edit the image following their intention and post to show around.
Furthermore, as phones and tablets becoming users' major mobile terminal, people prefer to take and edit photos on mobile devices, making it even more troublesome to edit and select regions on the small screen. 
Hence, automatic image editing guided by the user's voice input (\eg~Siri, Cortana) can significantly alleviate such problems. 
%% 2. PROBLEMS: describe the problem you want to tackle
% define the task of Language Driven Image Editing
% from low level retouching~\cite{hu2018exposure,yan2016automatic}, to high level semantic.
We research global image editing via language: given a source image and a language editing request, generate a new image transformed under this request, as firstly proposed in \cite{wang2018learning}.
% Note that the editing in this task could be accomplished by Photoshop human users on general images, which differs from other language-based image editing tasks~\cite{dong2017semantic,nam2018text,mao2019bilinear,li2019manigan} that focus on semantic manipulation on domain-specific images.
% However, there is no open dataset that supports LDIE task, so we collected a new dataset called FiveKReq, labeled with 24.75k image pairs with corresponding editing request.
%
%% 3. CHALLENGES: Describe the challenges. Briefly mention existing works, how they address the challenges, and what they fail to address.
% current work: 1. rule based method, focus on the improvement
% current work: 2. GAN based method, focus on the improvement for the fusion of  2. No specific way to do so.
Such a task is challenging because the model has to not only understand the language but also edit the image with high fidelity. 
Rule-based methods~\cite{manuvinakurike2018conversational,manuvinakurike2018edit} transfer the language request into sentence templates and further map the templates into a sequence of executable editing operations.
However, they require additional language annotations and suffer from unspecific editing requests. 
\cite{shi2020benchmark} directly maps the language to operations with the capability to accept the vague editing request, yet still need the operation annotation for training.
A more prevalent track is the GAN-based method~\cite{wang2018learning}, which models the visual and textual information by inserting the image and language features into a neural network generator that directly outputs the edited image. 
% It takes the advantages of strong modeling ability of neural network and good applicability to any domain image.
However, GAN-based models lack the interpretability about how an image was edited through a sequence of common editing operations (\eg~tone, brightness). Thus, they fail to allow users to modify the editing results interactively.
Moreover, GANs struggle with high-resolution images and is data-hungry.
% Since the network generator is largely over-parameterized, and the regular editing operations such as ``brightness'' are frequently used by human when do image retouching, grasp the prior knowledge of the possible editing operations could help the network to learn better, as indicated in~\cite{wang2018learning}.
% Hence, it still remains an open question about how to combine both the advantage of the strong modeling ability of neural network, and the modeling efficiency of predefined editing operations.

% \begin{figure}[!tp] 
%   \centering\includegraphics[width=1\columnwidth]{pics/teasing.pdf} 
%   \caption{Examples for LDIE task. \xn{this figure does not look so good. for a given input image and a request we can show our editing sequence compared to GAN based results. The input image could be high resolution.}} 
%   \label{fig:teasing}
% \end{figure} 

\begin{figure}[t] 
  \centering\includegraphics[width=1\columnwidth]{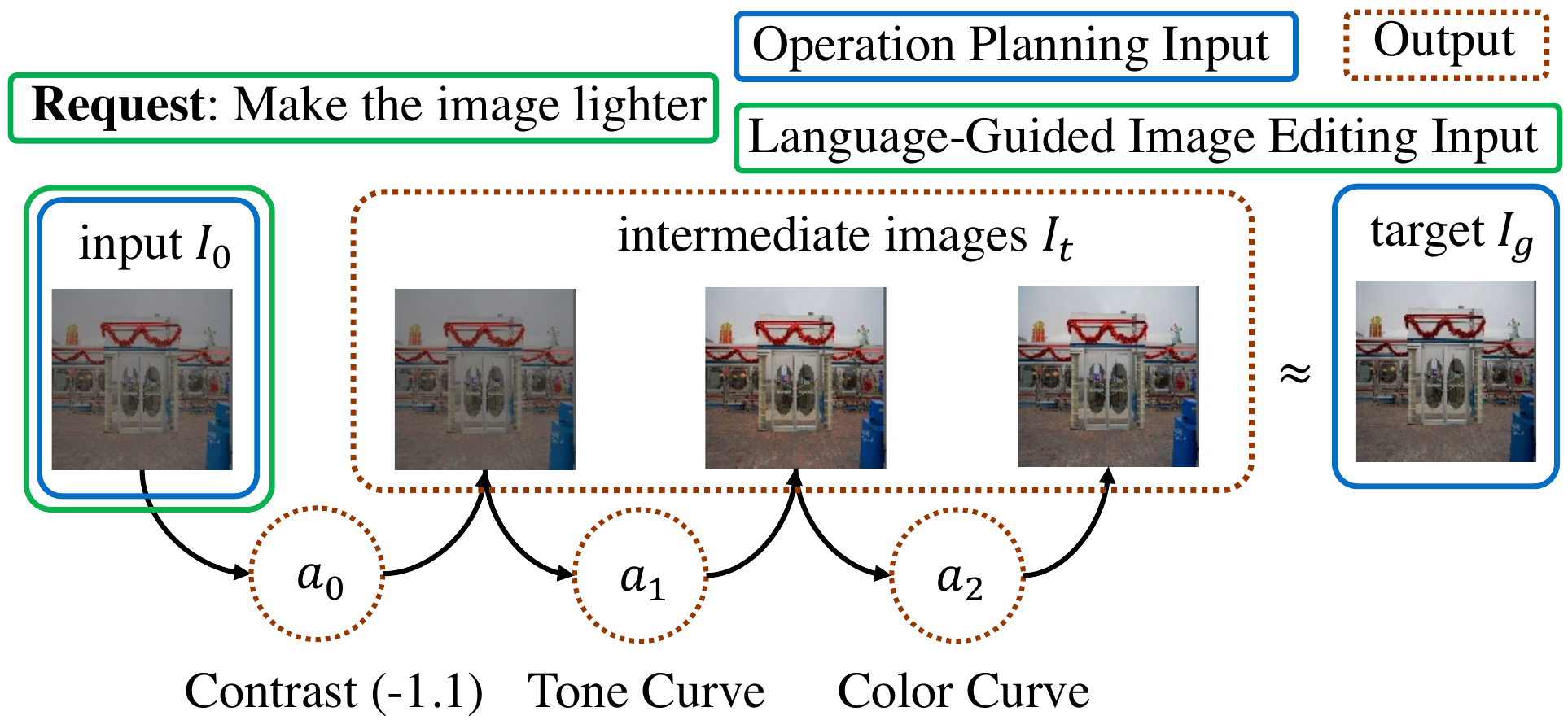} 
  \caption{Language-Guided Global Image Editing: given the input image $I_0$ and the request, we predict a sequence of actions $a_t$ to edit the image progressively with a series of intermediate images $I_t$ generated. And the final edited image is our output, which should accord with the request.
  Operation Planning: the input image $I_0$ and target image $I_g$ are given, and we plan a sequence of action to make the final edited image reach the target image $I_g$.}
  \label{fig:planning}
  \vspace{-3mm}
\end{figure} 

To provide an interpretable yet practical method for language-guided global image editing, in this paper, we propose a Text-to-Operation Network~(T2ONet). 
The network sequentially selects the best operations from a set of predefined everyday editing operations to edit the image progressively according to the language's comprehension and the visual editing feedback. 
As the operations are resolution-independent, such method will not deteriorate the image resolution.
Fig.~\ref{fig:planning} shows the process of mimicking human experts for professional photo editing and opens the possibility for human-computer interactions in future work.
% As a further step to answer the question, we still use a set of predefined editing operations to edit the image, but resort to a neural network to achieve the mapping from the input image and language to the series of operations.

One crucial difficulty for training our model is the lack of supervision information for editing sequences---we do not have access to intermediate editing operations and their parameters. The only available supervision is the input image's tuple, the target image, and the language editing request.
One possible solution is to train our model by Reinforcement Learning (RL). For example, the model can try different editing sequences and get rewards by comparing the edited images to the target images. However, it is well-known that RL is highly sensitive to hyper-parameters and hard to train when the action space is large (\eg high-dimensional continuous action).
On the other hand, it is demanding yet infeasible to collect annotations for all intermediate operations and their parameters in practice. Therefore, a novel training schema is expected to solve our task. 
%it is hard to optimize discrete operation selection by only relying on the target image.
To overcome this difficulty, we devise a weakly-supervised method to generate pseudo operation supervision. 
Inspired from the classical forward search planning~\cite{russell2016artificial}, we propose an operation-planning algorithm to search the sequence of operations with their parameters that can transform the input image into the target image, as shown in Fig.~\ref{fig:planning}. 
It works as an inverse engineering method to recover the editing procedure, given only the input and the edited images.
Such searched operations and parameters serve as pseudo supervision for our T2ONet.
Also, as the target image is used as the pixel-level supervision, we prove its equivalence to RL.
Besides, we show the potential of the planning algorithm to be extended to local editing and used to edit a new image directly.
%% 5. NOVELTY: Contrast ours to existing works to highlight our contribution

% \xn{It will be good if we use one paragraph to briefly talk about the datasets.}
In summary, our contributions are fourfold. 
First, we propose T2ONet to predict interpretable editing operations for language-guided global image editing dynamically.
Second, we create an operation planning algorithm to obtain the operation and parameter sequence from the input and target images, where the planned sequences help train T2ONet effectively.
Third, a large-scale language-guided global image editing dataset MA5k-Req is collected. 
Fourth, we reveal the connection between pixel supervision and RL, demonstrating the superiority of our weakly-supervised method compared with RL and GAN-based methods on AM5k-Req and GIER~\cite{shi2020benchmark} datasets through both quantitative and qualitative experimental results. 

%% 6. ORGANIZATION: The rest of the paper is organized as the following.

\section{Related Work}
\label{sec:related}

% 1. Language-based image editing tasks
\noindent\textbf{Language-based image editing.} 
% talk about the previous task
Language-based image editing tasks can be categorized into one-turn and multi-turn editing. In one-turn editing, the editing is usually done in one step with a single sentence~\cite{dong2017semantic,nam2018text,mao2019bilinear,li2020manigan}.
Dong~\etal~\cite{dong2017semantic} proposed a GAN-based encoder-decoder structure to address the problem. 
Nam~\etal~\cite{nam2018text} leverage the similar generator structure but use a text-adaptive discriminator to guide the generator in the more detailed word-level signal.
However, both \cite{dong2017semantic,nam2018text} simply use concatenation to fuse the textual and visual modalities. 
Mao~\etal~\cite{mao2019bilinear} proposes the bilinear residual layer to merge two modalities to explore second-order correlation.
Li~\etal~\cite{li2020manigan} further introduces a text-image affine combination module to select text-relevant area for automatic editing and use the detail correction module to refine the attributes and contents.
% The above works ara mainly to semantically manipulate the attributes of an image with a single salient object according to an image caption while preserving the text-irrelevant part.
However, the above works are built on the ``black box'' GAN model and inherit its limitations.
Shi~\etal~\cite{shi2020benchmark} introduces a new language-guided image editing (LDIE) task that edits by using interpretable editing operations, but its training requires the annotation of the operation.

For multi-turn editing, the editing request is given iteratively in a dialogue, and the edit should take place before the next request comes~\cite{el2019tell,cheng2018sequential}. However, only toy datasets are proposed for this task.

Our task belongs to a variant of one-turn editing that focuses on global image editing, which is proposed in~\cite{wang2018learning}, which also uses a GAN-based method by augmenting the image-to-image structure~\cite{isola2017image} with language input.
Different from all the above, our method can edit with complex language and image via understandable editing operations without the need for operation annotations

% \noindent\textbf{Language-Based Image Editing Methods.}
% % talk about the previous methods, compare rule based and non-rule based.
% Rule-based methods~\cite{manuvinakurike2018conversational,manuvinakurike2018edit} predefine the language template to parse the editing request, such as predicate and noun, and call the off-the-shelf operations.
% However, they cannot decide to what extent is the best to apply the parsed editing.

% 3. RL based method
\noindent\textbf{Image editing with reinforcement learning.}
To enable interpretable editing, \cite{hu2018exposure} introduces a reinforcement learning (RL) framework with known editing operations for automatic image retouching trained from unpaired images.
However, it cannot be controlled by language requests.

% 2. Planning
\noindent\textbf{Task planning.}
Task planning aims at scheduling a sequence of task-level actions from the initial state to the target state. 
Most related literature focuses on the pre-defined planning domain through symbolic representation~\cite{mcdermott1998pddl,ghallab2016automated,konidaris2018skills}.
Our \textit{operation planning} is reminiscent of task planning\cite{russell2016artificial}.
However, it is hard to use symbolic representation in our case because of high-dimensional states and continuous action space.

% 3. Modular network
\noindent\textbf{Modular networks.}
The modular networks are widely adopted in VQA~\cite{andreas2016neural,hu2017learning,johnson2017inferring,hu2018explainable,yi2018neural,Mao2019NeuroSymbolic} and Visual Grounding~\cite{hu2017modeling,liu2019learning,yu2018mattnet}.
In the VQA task, the question is parsed into a structured program, and each function in the program is a modular network that works specifically for a sub-task. 
The reasoning procedure thus becomes the execution of the program. However, the parser has discrete output, and it is usually trained with program semi-supervision~\cite{hu2017learning,johnson2017inferring} or with only the final supervision in an RL fashion~\cite{Mao2019NeuroSymbolic}.
LDIE task has a similar setting that only the target image is given as supervision, but we facilitate our model training by our planning algorithm.

%------ seq2seq figure 
\begin{figure*}[!tp] 
\centering\includegraphics[width=2\columnwidth]{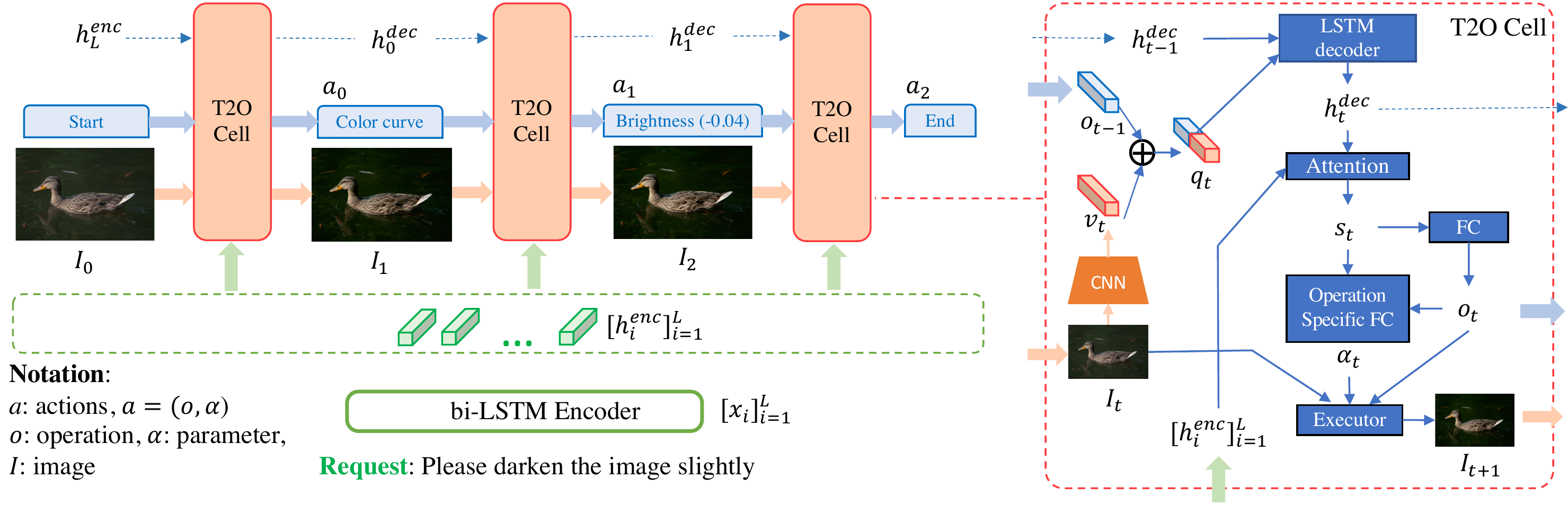} 
\caption{Structure of the T2ONet. An LSTM encoder embeds the request, and the T2O-Cell progressively decodes the input image and request into action and image series.  At each step $t$, the T2O-Cell generates the next action $a_{t}$ and image $I_{t+1}$ based on previous operation $o_{t-1}$, hidden state $h^{dec}_{t-1}$, and image $I_t$.}
\label{fig:seq2seq}
\end{figure*} 

\section{Method}
\label{sec:method}

We achieve the language-guided image editing by mapping the editing request into a sequence of editing operations, conditioned on both input image and language. 
We propose T2ONet to achieve such mapping (Sec.~\ref{sec:seq2seqModel}).
The critical difficulty is that we only have the target image's supervision but no supervision of the sequence.  
To tackle this difficulty, we introduce the idea of planning into the modeling to obtain a feasible operation sequence as the pseudo ground truth (Sec.~\ref{sec:editPlan}).
Finally, we talk about the training process (Sec.~\ref{sec:train}) and the connection to RL (Sec.~\ref{sec:RL}).
\subsection{Problem Formulation}
Starting with an input image $I_0$ and a language request $Q$, the goal is to predict an output image similar to the target image $I_g$.
In contrast to the GAN-based model, which outputs the edited image in one step, we formulate the editing problem through a sequential prediction of 
action sequence $\{a_t\}_{t=0}^{T}$ with length $T+1$ to edit the input image following the language request.
Applying $a_{t}$ to $I_t$ leads to $I_{t+1}$, and the final action $a_T$ is END action that will not produce new image, as shown in Fig.~\ref{fig:seq2seq}. In this way, the model generates a sequence of images $\{I_t\}_{t=1}^T$, where $I_T$ is the final output or target image. 
An action is defined as $a=(o, \alpha)$, where $o$ is the choice of discrete editing operations, and $\alpha$ is the continuous parameter of the operation. 

\subsection{Operation Implementation}
\label{sec:method:operation}

We adopt six operations: \textit{brightness}, \textit{saturation}, \textit{contrast}, \textit{sharpness}, \textit{tone}, and \textit{color}.
Among them, \textit{brightness} and \textit{saturation} is implemented by scaling H and S channels in the HSV space~\cite{gonzales2002digital}, controlled by a single re-scaling parameter. 
\textit{Sharpness} is implemented by augmenting the image with spatial gradients, controlled by a single parameter. 
\textit{Contrast} is also a single-parameter operation and implemented following~\cite{hu2018exposure}. 
\textit{Tone} is controlled by eight parameters that construct a pixel value mapping curve, following~\cite{hu2018exposure}. 
Finally, \textit{color} is similar to \textit{tone} but is implemented with three curves that operate on each of RGB channels, each controlled by eight parameters.
The details of the operation implementation are in the Appx.~\ref{appx:op_detail}.

\subsection{The Text-to-Operation Network (T2ONet)}
\label{sec:seq2seqModel}
% \Jing{Describe the model for data flow. Also explain the intuition of each part of the model, 1. why input both image and sentence. 2. why the intermediate state needs previous image and previous actions. Why not use parameter (Should demonstrate it with experiment.) 3. attention.} 
We propose the T2ONet to map the language request and the input image to a sequence of actions, which optimizes the joint action distribution, where each new action is predicted based on its past actions and intermediate images:
\begin{align}
  P(\{a_t\}_{t=0}^{T}|I_0, Q) &= P(a_0|I_0, Q) \times \nonumber\\ &\prod_{t=1}^TP(a_t|\{a_\tau\}_{\tau=0}^{t-1}, \{I_\tau\}_{\tau=0}^{t}, Q).
\end{align}
 We denote state $s_t$ as the condensed representation of $\big(\{a_\tau\}_{\tau=0}^{t-1}, \{I_\tau\}_{\tau=0}^{t}, Q\big)$, then the objective is transformed to:  
  $P(\{a_t\}_{t=0}^T|s_0) = \prod_{t=0}^TP(a_t|s_t)$\;.
To realize the policy function  $P(a_t|s_t)$, we adopt an Encoder-Decoder LSTM architecture~\cite{cho2014learning}, shown in Fig.~\ref{fig:seq2seq}.
The request $Q=\{x_i\}_{i=1}^L$ is encoded using a bi-directional LSTM upon the GloVe word embeddings~\cite{pennington2014glove} into a series of hidden states $\{h_i^{enc}\}_{i=1}^L$ and the final cell state $m_L^{enc}$.
Then, an LSTM decoder is represented as $h_{t+1}^{dec}, m_{t+1}^{dec} = f(h_t^{dec}, m_t^{dec}, q_t)$,
where $q_t=\mathrm{concat}(\textrm{Embedding}(o_t); v_t)$. $o_t$, $h_t^{dec}$, and $m_t^{dec}$ are the predicted operation, the hidden state, and the cell state at the $t$-th step, respectively (we omit $m_t^{dec}$ in Fig.~\ref{fig:seq2seq} for simplicity). Similar to word embedding, each operation is embedded into a feature vector through a learnable operation embedding layer. $v_t=\textrm{CNN}(I_t)$ denotes the image embedding via CNN at the $t$-th step. 
 Then, the attention mechanism~\cite{bahdanau2014neural} is applied to better comprehend the language request 
  $\beta_{ti} = \frac{\exp{((h_t^{dec})^\mathrm{T}h_i^{enc})}}{\sum_{i'=1}^L\exp{((h_t^{dec})^\mathrm{T}h_{i'}^{enc})}}$,
  $c_t = \sum_{i=1}\beta_{ti}h_i^{enc}$,
  $s_t = \tanh(W_c[c_t;h_t^{dec}])$.
The state vector $s_t$ is now the mixed feature of past images, operations, and the language request. 
Since the parameter $\alpha$ is dependent on the operation $o$, we further decompose the policy function as 
$P(a_t|s_t)=P(o_t, \alpha_t|s_t)=P(o_t|s_t)P(\alpha_t|o_t, s_t)$,
where $P(o_t|s_t$) is obtained through a Fully-Connected~(FC) layer to predict the operation $o_t$, which is expressed as:
\begin{equation}
  P(o_t|s_t) = \textrm{softmax}(W_os_t + b_o).\label{eqn:op_prob}
\end{equation}
For parameter prediction $P(\alpha_t|o_t, s_t)$, different operations can have different parameter dimensions. Therefore, we create an operation-specific FC layer for each operation to calculate: $\alpha_t = W_\alpha^{(o)} s_t + b_\alpha^{(o)}$, where superscription $(o)$ is the indicator of the specific FC layer for operation $o$. 
Hence, $P(\alpha_t|o_t, s_t)$ is modeled as a Gaussian distribution $\mathcal{N}(\alpha_t; \mu_{\alpha_t}, \sigma_\alpha)$: 
\begin{equation}
  P(\alpha_t|o_t, s_t) = \mathcal{N}(\alpha_t; W_\alpha^{(o_t)} s_t + b_\alpha^{(o_t)}, \sigma_\alpha).
  \label{eqn:param_prob}
\end{equation}
Finally, the executor will apply the operation $o_t$ and its parameter $\alpha_t$ to the image $I_{t}$ to obtain the new image $I_{t+1}$. 
The process from $I_{t}$ to $I_{t+1}$ will repeat until the operation is predicted as the ``END'' token.

% analysis about the model design.
% \xn{Although the LSTM hidden states contain previous editing information, we find in our ablation study that it is still beneficial to explicitly take as inputs the previous image feature and previous operation embedding.} 
% We find in the ablation study that it is beneficial to take as inputs the previous image feature and previous operation embedding at each decoding step, 
% because the previous image help the model comprehend how much the image is edited so far, and the 
% the historical operation information helps the T2OCell decide which operation to choose based on the previously applied operations to avoid repetitive editing using the same operation.
% \xn{Can't understand the following point very well, maybe remove it. Or do you try to say ``our executor is involved in the training phase which can help training"? but why?} Different from the usual case in VQA~\cite{yi2018neural,Mao2019NeuroSymbolic} that the question parser and executor are separated into two stages, our T2ONet conducts mapping from input language to operations interactively with operation execution, allowing richer states information to predict the further action.

% operation planning
\begin{algorithm}[!tp]
  \label{alg:forwardSearch}
\SetAlgoLined
\LinesNumbered
\KwIn{$I_0$, $I_{g}$, max operation step $N$, threshold $\epsilon$, beamsize $B$, operation set $\mathcal{O}$}
$p$=$[I_0]$\\
$\mrm{cost}(I) = ||I - I_g||_1$\\
\For{$t$ in $1:N$}{
  $q \leftarrow [\ ]$\\
  \For{$I \in p$}{
    \For{$o \in \mathcal{O}$}{
      $\alpha^* = \arg\min_\alpha \mrm{cost}(o(I, \alpha))$\\
      $I^* \leftarrow o(I, \alpha^*)$\\
      $q \leftarrow q\cup I^*$\\
    }
  }
  $q\leftarrow \mrm{Sort}(q), \mrm{sort key}=\mrm{cost}(I^*)$ \\
  $p = q[:B]$\\
  \For{$I \in p$}{
    \If{$\mrm{cost}(I) < \epsilon$}{
      Break All Loop
    }
  }
}
$\{o_t\}, \{\alpha_t\}, \{I_t\}\leftarrow \mrm{Backtracking}(p)$

\Return $\{o_t\}, \{\alpha_t\}, \{I_t\}$
 \caption{Operation Planning}
\end{algorithm}

\subsection{Operation Planning}
\label{sec:editPlan}
% \Jing{Describe the inputs and say it is a sub problem. And try to give the formal definition of planning. 
% If not, must illustrate the intuition behind it, and with enough citations, saying similar to some other papers' method.}
% \Jing{Firstly show the algorithm. Then explain each part of the algorithm and design idea. 1. why you optimize the parameter. 2. The time and efficiency of the algorithm, compare the time complexity and compare the time of such operation over differen optimizer with the trade-off about the accuracy. 3. Compare different operation list and discuss whether the selection of the order of the operation will cause some effect.}

To provide stronger supervision for training policy function, we introduce the operation planning algorithm that can reverse engineer high-quality action sequences from only the input and target images. 
Concretely, given the input image $I_0$ and the target image $I_g$, plan an action sequence $\{a_t\}_0^T$ to transform $I_0$ into $I_g$. 
This task is similar to the classical planning problem~\cite{ghallab2016automated}, and we solve it with the idea of forward-search. 
Algorithm~\ref{alg:forwardSearch} shows the operation planning process. 
We define the planning model with action $a$, image $I$ as state, and state-transition function $I' = o(I, \alpha)$, where $o$ is the operation. 
The state transition function takes image $I$ and parameter $\alpha$ as input and outputs a new image.
The goal is to make the final image $I_T$ similar to $I_g$ as within an error $\epsilon$, specified by the L1 distance $||I_T-I_g||_1 < \epsilon$. To reduce redundant edits, we restrict each operation to be only used once and limit the maximum edit step to $N$. 

In algorithm~\ref{alg:forwardSearch}, we wrap the goal into a cost function and try to minimize the cost during each step. However, the action $a$ includes both discrete operation $o$ and continuous parameter $\alpha$, which could be high-dimensional with extremely large searching space. To make computing efficient, we only loop over all the discrete operation candidates, but as the operation is chosen, we optimize the parameter to minimize the cost function. Such optimization could significantly reduce the searching space for parameters. Since all operations here are differentiable, the optimization process could be 0th-, 1st-, and 2nd-order, \eg, Nelder-Mead~\cite{nelder1965simplex}, Adam~\cite{kingma2014adam}, and Newton's method, respectively.
At each step $t$, the algorithm visits every image in the image candidate list of beam size $B$, and for each image, the algorithm enumerates the operation list of size $|\mathcal{O}|$. 
Since it has at most $N$ steps, the maximum time complexity for operation planning is $O(NB|\mathcal{O}|)$.
In practice, we constraint the planning for unrepeated operations.
Fig.~\ref{fig:vis_plan} shows one trajectory of our planned sequence, as it stops at the second step since the cost is lower than $\epsilon=0.01$.
Different operation sets and orders are studied in Sec.~\ref{sec:ablation}.
We further show two potential extensions of the operation planning algorithm.

\noindent\textbf{Extension1: Planning through a discriminator}.
The cost($I$) is not limited to $||I_T-I_g||_1$, but can be the image quality score yield by a pretrained discriminator $D$ without dependence of the target image.
Then our operation planning can directly edit new images (see Sec.~\ref{sec:discussion} for details).

\noindent\textbf{Extension2: Planning for local editing}.
Although our paper focuses on global editing, the operation planning can be extended to planning local editing by searching the region masks with an additional loop, detailed in Sec.~\ref{sec:discussion}.

%------------------------------

% --- figure for planning ---
\begin{figure}[!tp] 
  \centering\includegraphics[width=\columnwidth]{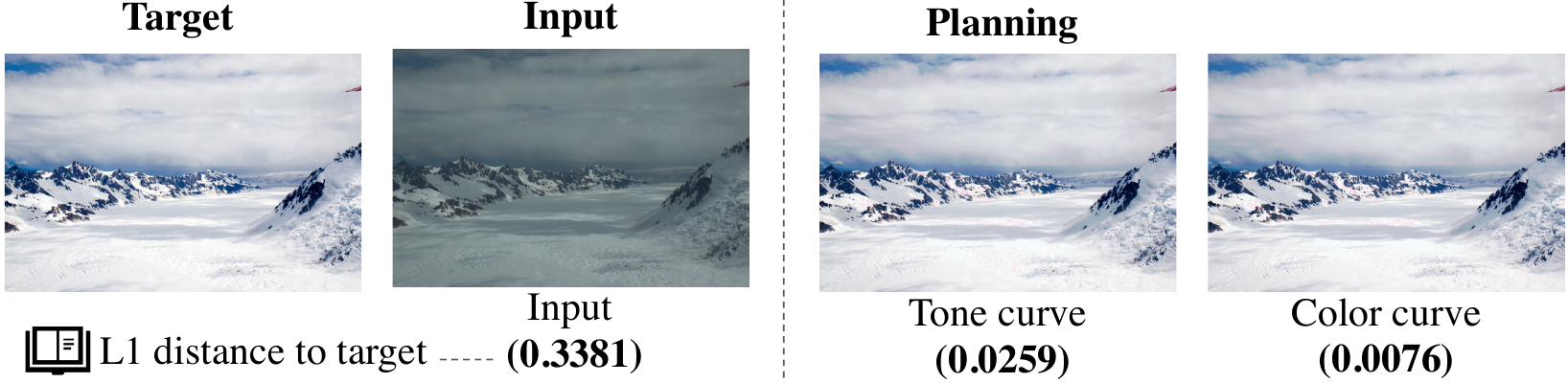} 
  \caption{Visualization of the operation planning trajectory. The number
  The L1 distance is monotonically decreasing and can recover highly similar result to the target.} 
  \label{fig:vis_plan}
\vspace{-4mm}
\end{figure} 

\subsection{Training}
\label{sec:train}
The planning algorithm~\ref{alg:forwardSearch} creates pseudo ground truth operation $\{o_t^*\}_{t=0}^T$  and parameter sequence $\{\alpha_t^*\}_{t=0}^{T-1}$ to supervise our model. The operation is optimized by minimizing the cross-entropy loss (XE): 
\begin{equation}
  \vspace{-2mm}
  \mathcal{L}_o = -\sum_{t=0}^T\log(P(o_t^*|s_t)).
\end{equation}
Maximizing the log-likelihood for Eq.~\ref{eqn:param_prob} equals to applying  MSE loss:
\begin{equation}
  \mathcal{L}_\alpha = \sum_{t=0}^{T-1}||\alpha_t - \alpha_t^*||_2^2.
\end{equation}
Additionally, to utilize the target image supervision, we apply the image loss as final L1 loss as: 
\begin{equation}
  \mathcal{L}_{L1} = ||I_T - I_g||_1.
\end{equation}
The ablation study (Appx.~\ref{appx:diff_img_loss}) proves the L1 loss is critical for better performance.
Although teacher forcing technique is a common training strategy in sequence-to-sequence model~\cite{sutskever2014sequence}, where the target token is passed as the next input to the decoder, teacher forcing does not work for $\mathcal{L}_{L1}$ since the intermediate pseudo-GT input blocks the gradient.
Therefore we train $\mathcal{L}_{L1}$ in a non-teacher forcing fashion and $\mcal{L}_o, \mcal{L_\alpha}$ in the teacher forcing fashion, alternatively.
Our final loss is $\mathcal{L}=\mathcal{L}_o + \mathcal{L}_{\alpha} + \mathcal{L}_{L1}$.

\noindent\textbf{More request-sensitive output}.
The model is expected to be request-sensitive: produce diversified edits following different requests, rather than simply improve the image quality regardless of the requests.
To improve the request-sensitivity, we propose to sample the parameter $\alpha_t$ from $\mcal{N}(\alpha_t;\mu_{\alpha_t}, \sigma_\alpha)$ in Eq.~(\ref{eqn:param_prob}) to train the image loss.
In our default setting, $\sigma_\alpha=0$, \ie $\alpha_t=\mu_{\alpha_t}$. 
Our motivation is that sampling the parameter will produce stochastic editing results, preventing the model from falling into one same editing pattern or shortcuts regardless of the language.
Also, there exist multiple reasonable edits for one request, so the $\mcal{L}_{L1}$ still guarantees the stochastic output images to be reasonable.
We observe that increasing $\sigma_\alpha$ leads to higher request-sensitivity (see Sec.~\ref{sec:ablation}).
In fact, the next section will discuss the above training scheme for image loss with a close relation with RL.

\subsection{Equivalence of Image Loss and DPG}
\label{sec:RL}

To bridge the equivalence, we adapt an RL baseline from \cite{hu2018exposure}. Due to space limitations, the detailed introduction of the baseline is in Appx.~\ref{appx:rl_baseline}, here we focus on the training for parameter $\alpha$ with RL and its connection to image loss.
Let the reward be $r_t = \mrm{cost}(I_{t-1}) - \mrm{cost}(I_{t})$, policy $\pi_o=P(o|s)$ in Eq.~(\ref{eqn:op_prob}), $\pi_\alpha=\mcal{N}(\alpha;\mu_{\alpha}, \sigma_\alpha)$, 
the accumulated reward defined as $G_t = \sum_{\tau=0}^{T-t}\gamma^{\tau}r_{t+\tau}$ ($\gamma=1$ as \cite{hu2018exposure}), the goal is to optimize the objective $J(\pi)=\mbb E_{(I_0, Q)\sim P(\mcal D), o\sim\pi_o \alpha\sim \pi_\alpha } G_1$. 
The continuous policy $\pi_\alpha$ is optimized by Deterministic Policy Gradient algorithm (DPG)~\cite{silver2014deterministic}.
Different from the common setting \cite{silver2014deterministic,hu2018exposure} where the Q function is approximated with a neural network to make it differentiable to action, we approximate $Q$ as $G$ since our $G_{t+1}$ is already differentiable to $\alpha_t$, resulting in the DPG for each episode as

% where $\mrm{cost}(I)$ can be any image loss and is set as $||I-I_g||_1$ in our experiment. 
% % Since the reward for the ``END'' action is hard to design, we set every episode fixed $T$ steps ($T=5$ as \cite{hu2018exposure}).
% The actions are sampled from policy $\pi_o$ and $\pi_\alpha$ , where $\pi_o=P(o|s)$ in Eq.~(\ref{eqn:op_prob}), $\pi_\alpha=\mcal{N}(\alpha;\mu_{\alpha}, \sigma_\alpha)$, leading to the trajectory $\Pi=\{s_0, a_0, s_1, r_1, ..., s_T, r_T\}$. 
% With the accumulated reward defined as $G_t = \sum_{\tau=0}^{T-t}\gamma^{\tau}r_{t+\tau}$ ($\gamma=1$ as \cite{hu2018exposure}), the goal is to optimize the objective $J(\pi)=\mbb E_{(I_0, Q)\sim P(\mcal D), \Pi\sim \pi} G_1$. 
% Similar to \cite{hu2018explainable}, the discrete policy $\pi_o$ is optimized via REINFORCE~\cite{williams1992simple} and the continuous policy $\pi_\alpha$ is optimized by Deterministic Policy Gradient algorithm (DPG)~\cite{silver2014deterministic}. 
% Different from the common setting \cite{silver2014deterministic,hu2018exposure} where the Q function is approximated with a neural network to make it differentiable to action, we approximate $Q$ as $G$ since our $G_{t+1}$ is already differentiable to $\alpha_t$, resulting in the DPG for each episode as 
\vspace{-2mm}
\begin{equation}
  \nabla_{\theta_\alpha}J(\pi) = \underset{\mbb E}
  \sum_{t=0}^{T-1}\nabla_{\alpha_t}G_{t+1} \nabla_{\theta_\alpha}\alpha_t. 
  \label{eqn:alpha_grad}
\end{equation} 
Now, we show the equivalence between image loss and DPG using the following theorem:
\begin{theorem}
  The DPG for $\alpha$ in Eq.~(\ref{eqn:alpha_grad}) can be rewritten as 
  \label{thm:dpg}
\begin{equation}
  \nabla_{\theta_\alpha}J(\pi) = -
 \frac{\p\mrm{cost}(I_T)}{\p \theta_\alpha}.\label{eqn:equality}
\end{equation}
\end{theorem}
\begin{proof}
  See Appx.~\ref{appx:eq_proof}
\end{proof}
Theorem~\ref{thm:dpg} provides a new perspective that minimizing the $\mcal{L}_{L1}$ for the final image in T2ONet is actually equivalent to optimizing the model with deterministic policy gradient at each step.

\section{Experiments}
\subsection{Datasets}

\noindent\textbf{MA5k-Req.} To push the research edge forward, we create a large-scale language-guided global image editing dataset. We annotate language editing requests based on MIT-Adobe 5k dataset~\cite{bychkovsky2011learning}, where each source image has five different edits by five Photoshop experts, leading to a new dataset called MA5k-Req.
4,950 unique source images are selected, and each of the five edits is annotated with one language request, leading to 24,750 source-target-language triplets. 
See Appx.~\ref{appx:data_collect} for data collection details.
We split the dataset as 17,325~(70\%) for training, 2,475~(10\%) for validation, and 4,950~(20\%) for testing.
After filtering the words occurring less than 2 times, the vocabulary size is 918.
Note that \cite{wang2018learning} also similarly creates a dataset with 1884 triplets for this task, but unfortunately, it has not been released and is 10 times smaller than ours.

\noindent\textbf{GIER.} Recently, GIER dataset~\cite{shi2020benchmark} is introduced with both global and local editing. We only select the global editing samples, leading to a total of 4,721 unique image pairs, where each is annotated with around 5 language requests, resulting in 23,171 triplets. we splits them as 18,571 (80\%) for training, 2,404 (10\%) for validation, and 2,196 (10\%) for testing.
After filtering the words occurring less than 3 times, the vocabulary size is 2,102.

% Motivation to collect such data: 1. There is no general image data pair with such ability. list few of them, but with defeats.
% To enable with Current datasets for LBIE~\cite{chen2018language,nam2018text,mao2019bilinear} are mainly limited to domain-specific data, e.g., bird~\cite{wah2011caltech}, flower~\cite{nilsback2008automated}, and fashion~\cite{zhu2017your}, and have no paired image samples for training.
% Collecting process: brief
% quality control (challenge)
% Some images sampled from the two dataset are shown in Fig.~\ref{fig:dataset}.
% The most related dataset to our work is \cite{wang2018learning}. 
% However, it only has 1884 triplets and is not released.

% \noindent\textbf{FiveKReq.}\quad FiveKReq is spited into train/val/test for 17325/2475/4950 image pairs.

% \noindent\textbf{GIER.}\quad GIER contains from real user image editing request on user photos with both global and local editing. It contains 6179 unique image pairs where each pair is annotated with 5 language requests. The dataset is splits into 4934/618/618 image pairs for train/val/test split.

%-----------Table: Major Result ----------------
\newcommand{\ra}[1]{\renewcommand{\arraystretch}{#1}}
\begin{table*}[t]\centering
\ra{1.2}
\scalebox{0.82}{
\begin{tabular}{@{}lrrrrrcrrrrr@{}}
\toprule
& \multicolumn{5}{c}{MA5k-Req} & \phantom{ab}& \multicolumn{5}{c}{GIER}\\
\cmidrule{2-6} \cmidrule{8-12} 
& L1 $\downarrow$ & SSIM$\uparrow$ & FID$\downarrow$ & $\sigma_{\times 10^2}$$\uparrow$ & User$\uparrow$ && L1$\downarrow$ & SSIM$\uparrow$ & FID$\downarrow$ &  $\sigma_{\times 10^2}$$\uparrow$ & User$\uparrow$\\ 
\midrule
Target &  -   &   -    &   -     & -  & 3.5053  &&  -     &   -     &   -   & -  & 3.6331 \\
Input & 0.1190 & 0.7992 & 12.3714 & - & - && 0.1079 & 0.8048 & 49.6229 & - & - \\
Bilinear GAN~\cite{mao2019bilinear} & 0.1559 & 0.4988 & 102.1330 & 0.8031 & 1.9468 && 0.1918 & 0.4395 & 214.7331 &  1.2164 & 1.7988\\
% FilterBank\cite{wang2018learning} & 0.1114 &  0.7299 & 49.2135 & 0.5191 & 2.4415 && 0.1219 & 0.7123 & 132.0352 & 0.0146 & 2.1893 \\
Pix2pixAug~\cite{wang2018learning} & 0.0928 &  0.7938 & 14.5538 & 0.5401 & 3.0957 && 0.1255 & 0.7293 & 74.7761 & \textbf{1.2251} & 2.5148\\
SISGAN~\cite{dong2017semantic} & 0.0979 & 0.7938 & 30.9877 & 0.1659 & 2.8032 && 0.1180 & 0.7300 & 140.1495 & 0.0198 & 2.1243\\
TAGAN~\cite{nam2018text} & 0.1335 & 0.5429 & 43.9463 & 1.5552 & 2.5691 && 0.1202 & 0.5777 & 112.4168 & 0.6073 & 2.4970\\ 
% ManiGAN~\cite{li2020manigan} \\
GeNeVa~\cite{el2019tell} & 0.0933 & 0.7772 & 33.7366 & 0.6091 & 3.0851 &&  0.1093 & 0.7492 & 87.0128 & 0.5732& 2.7278 \\
RL & 0.1007 & 0.8283 & 7.4896 & \textbf{1.6175} & 3.1968 &&  0.2286 & 0.3832 & 132.1785 & 0.3978& 1.8462 \\
% T2ONet + RL & 0.0955 & 0.8330 & 7.1413 & 1.4672 & &&  0.1052 & 0.8075 & 49.4183 & 1.0949 & \\
T2ONet & \textbf{0.0784} & \textbf{0.8459} & \textbf{6.7571} & 0.7190 & \textbf{3.3830}  && \textbf{0.0997} & \textbf{0.8160} & \textbf{49.2049} & 0.6226 & \textbf{2.8994}\\
\bottomrule
\end{tabular}}
\caption{Quantitative results on two test sets. $\sigma_{\times 10^2}$ means that the image variance has been scaled up 100 times.}
\label{tab:mainResult}
\vspace{-4mm}
\end{table*}

\subsection{Evaluation Metrics}
Similar to the L2 distance used in ~\cite{wang2018learning}, we use L1 distance, Structural Similarity Index~(SSIM), and Fréchet Inception Distance~(FID) for evaluation. 
L1 distance directly measures the averaged pixel absolute difference between the generated image and ground truth image as the pixel range is normalized to 0-1.
SSIM measures image similarity through luminance, contrast, and structure.
FID measures the Fréchet distance between two Gaussians fitted to feature representations of the Inception network over the generated image set and ground truth image set. 
To further exam the model's language-sensitivity, we propose the image variance $\sigma$ to measure the diversity of the generated image conditioned on different requests. Similar to~\cite{li2019diverse}, we apply 10 different language requests (see Appx.~\ref{appx:lang}) to the same input image and output 10 different images. 
Then we compute the variance over the 10 images of all pixels and take the average overall spatial locations and color channels. 
Finally, we take the average of the average variance over the entire test set. 
The variance can only measure the diversity of generated images in different language conditions but cannot directly tell the editing quality. 
So we still resort to user study to further measure the editing quality.

\noindent\textbf{User study setting.} We randomly select 250 samples from the two datasets, respectively, with each sample evaluated twice.
The user will see the input image and request and blindly evaluate the images predicted by different methods as well as the target image.  
Each user rates a score from 1 (worst) to 5 (best) based on the edited image quality (fidelity and aesthetics) and whether the edit accords with the request. 
We collect the user rating through Amazon Mechanical Turk (AMT), involving 42 workers.

\subsection{Implementation Details}
For operation planning, we set the maximum step $N=6$, tolerance $\epsilon = 0.01$, and constraint that one operation is only used once.
We adopt Nelder-Mead~\cite{nelder1965simplex} for parameter optimization.
The model is optimized by Adam~\cite{kingma2014adam} with learning rate 0.001, $\beta_1=0.9$, $\beta_2=0.999$. 
More details are elaborated in Appx.~\ref{appx:imple_detail}.

\subsection{Main Results}
\label{sec:experiment_main_result}
\noindent\textbf{Operation planning}. 
The set 5 in Tab.~\ref{tab:dif_op_list} shows the averaged L1 distance of the planning result is 0.0136, which is around only 3.5-pixel value error towards target images, with pixel range 0-255.
Fig.~\ref{fig:vis_plan} shows the operation planning can achieve the visually indistinguishable output compared with the target.
So we are confident to use the planned action sequence as a good pseudo ground truth.

\noindent\textbf{Comparison methods}.
\begin{itemize}[noitemsep, topsep=0pt]
\item \textit{Input}: the evaluation between input and target image.
\item \textit{Bilinear GAN}~\cite{mao2019bilinear}, \textit{SISGAN}~\cite{dong2017semantic}, \textit{TAGAN}~\cite{nam2018text}: these three methods are trained  by learning the mapping between the caption and image without image pairs. Since there is not image caption in our task but the paired image and request, we drop the procedure of image-caption matching learning but adapt them with the L1 loss between input and target images.
% \item \textit{FilterBank}~\cite{wang2018learning}: the filter bank method used in \cite{wang2018learning}.
\item \textit{Pix2pixAug}~\cite{wang2018learning}: the pix2pix model~\cite{isola2017image} augmented with language used in \cite{wang2018learning}.
\item \textit{GeNeVa}~\cite{el2019tell}: a GAN-based dialogue guided image editing method. We use it for single-step generation.
\item \textit{RL}: out RL baseline introduced in Sec.~\ref{sec:RL}.
\end{itemize}
We also compared with ManiGAN~\cite{li2020manigan}, but its output is very blurred as it is not designed for our task, and its network lacks the skip connection structure to keep the resolution. 
So we just show its visualization in Appx.~\ref{appx:vis_comp}.
% \xn{how are the baselines implemented? do we want to talk about it briefly in the main paper or in supplementary material?}

\begin{figure*}[!tp] 
  \centering\includegraphics[width=2\columnwidth]{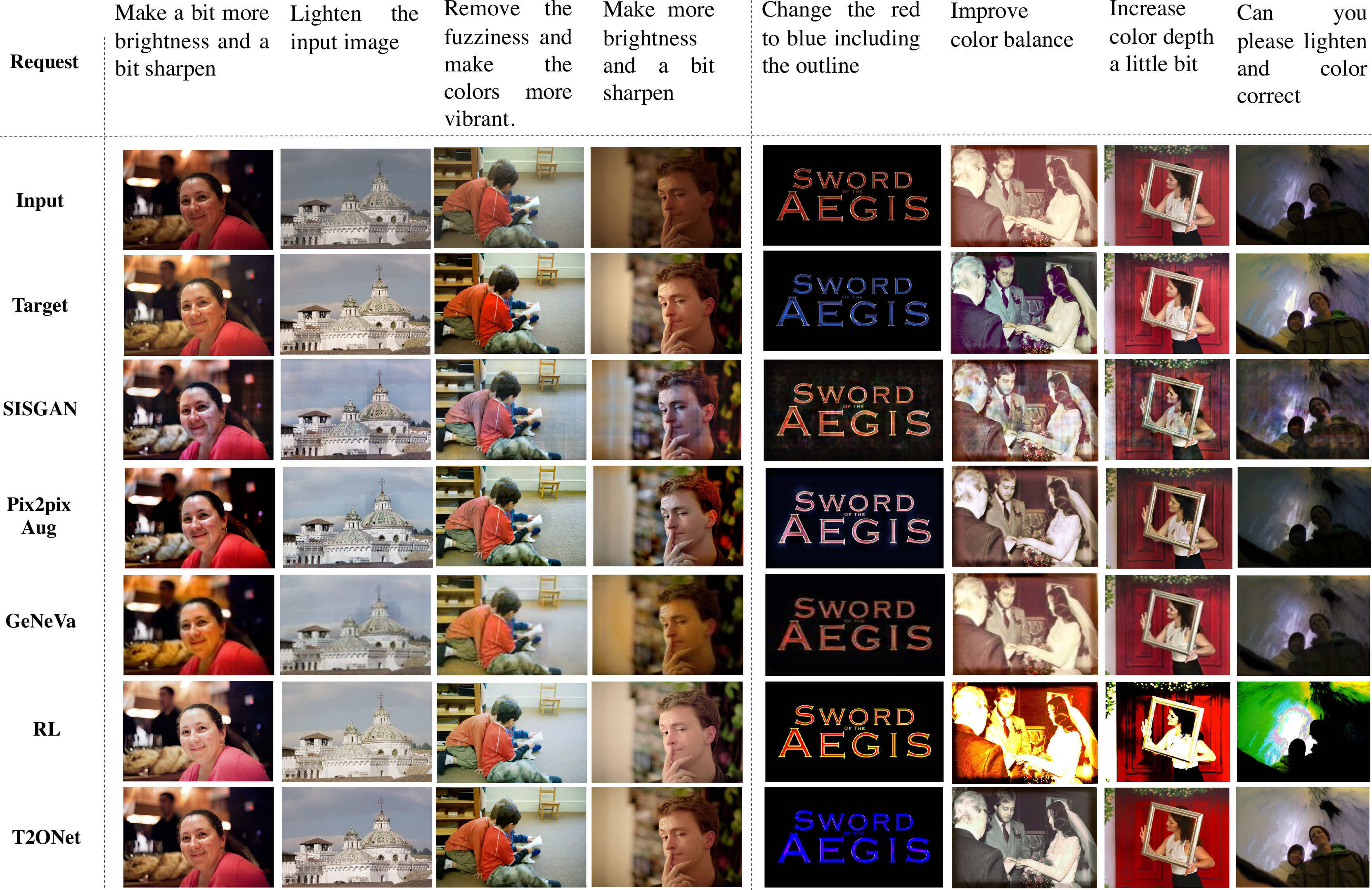} 
  \caption{Visualization for comparison of our method T2ONet with other methods on MA5k-Req (left) and GIER (right).}
  \label{fig:vis_comp_exp}
 \vspace{-4mm}
\end{figure*} 

% figure of trade-off table
\begin{figure}[!tp] 
  \centering\includegraphics[width=0.8\columnwidth]{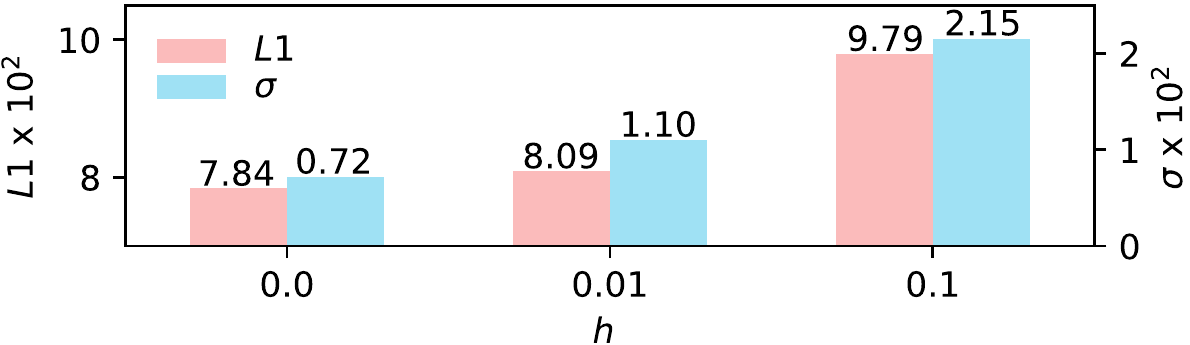} 
  \caption{L1 and variance trade-off by training with different parameter sample variance on the MA5k-Req test set.} 
  \label{fig:tradeoff}
  \vspace{-1mm}
\end{figure} 

\noindent\textbf{Result analysis}.
The qualitative and quantitative comparison are in Fig.~\ref{fig:vis_comp_exp} and Tab.~\ref{tab:mainResult}, respectively. 
However, the results of BilinearGAN, TAGAN are bad, and their visual results have been omitted. For interested readers please refer to Appx.~\ref{appx:vis_comp}.
Fig.~\ref{fig:vis_comp_exp} shows that
SISGAN has obvious artifacts, Pix2pixAug, and GeNeVa have less salient editing than ours, the RL tends to be overexposed in Fivek-Req and does not work well on GIER. 
Our T2ONet generates more aesthetics and realistic images, which are most similar to targets.
The much worse performance of BilinearGAN, TAGAN, SISGAN might because their task is different from ours and their model ability is limited for complex images.
Tab.~\ref{tab:mainResult} demonstrates that our T2ONet achieves the best performance on visual similarity metrics L1, SSIM, and FID, but not the $\sigma$.
Firstly, $\sigma$ can measure the editing diversity, as in Fig.~\ref{fig:vis_variance}; however, the $\sigma$ and visual similarity metric are usually a trade-off, as shown in Sec.~\ref{sec:ablation}.
So although RL has the highest $\sigma$ under MA5k-Req, it sacrifices L1 much more, and its visual results indicate that it tends to be overexposed.
Second, the $\sigma$ might be dominated by noisy random artifacts, \eg, BilinearGAN in Fig.~\ref{fig:vis_comp_exp}.
Therefore, we resort to user ratings for best judgment, which indicates our method is the most perceptually welcomed.

\noindent\textbf{Dataset Comparison}. Tab.~\ref{tab:mainResult} also reflects the difference between the two datasets.
Since GIER has a smaller data size and contains more complex editing requests, GIER is more challenging than MA5k-Req, which is verified by the fact that the gap of the user rating between target and T2ONet is much larger on GIER than on MA5k-Req.\\
\noindent\textbf{Advantage over GAN}.
GAN-based methods also suffer from high-resolution input and can be jeopardized by artifacts.
However, our T2ONet is resolution-independent without artifacts (see Appx.~\ref{appx:res_independ}).\\
\noindent\textbf{Advantage over RL}.
With the more challenging GIER dataset, it makes RL harder to explore the positive-rewarded actions and fail.
However, T2ONet still works well on GIER with the help of the pseudo action ground truth from operation planning.
We further show that the operation planning can help RL in Appx.~\label{appx:plan_help_rl}.

% figure of variation visualization
\begin{figure}[!tp] 
  \centering\includegraphics[width=\columnwidth]{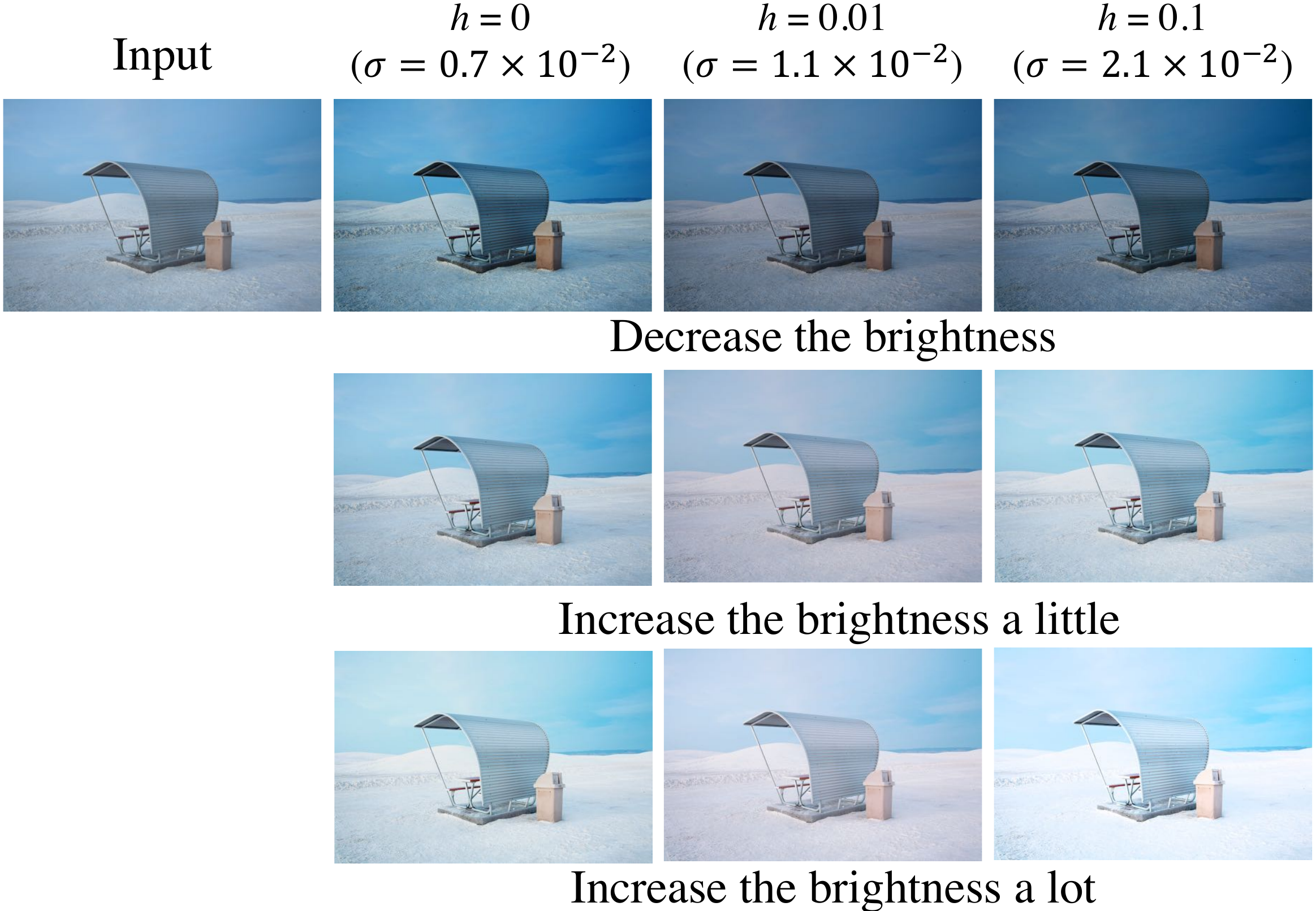} 
  \caption{The same input edited with different language by models trained with different $h$. Image variance $\sigma$ for the whole test data is also shown as a reference. The model trained with larger $h$ has more diversified output.} 
  \label{fig:vis_variance}
  \vspace{-6mm}
\end{figure} 

\begin{table}[t]
  \centering
  \scalebox{0.76}{
  \begin{tabular}{c|ccccc|c}
  \toprule
  operation set & 1 & 2 & 3 & 4 & 5 & input\\
  \midrule
  planning (train) & 0.0521 & 0.0358 & 0.0198 & 0.0197 & \textbf{0.0136} & 0.1202\\
  T2ONet (test) & 0.1315 & 0.0857 & 0.0832 & 0.0853 & \textbf{0.0770} & 0.1190\\
  \bottomrule
  \end{tabular}}
  \caption{L1 distance to target image over different operation lists and operation orders on MIT-Adobe 5k dataset. Set 1 is planned over only brightness operation. Set 2 is planned over single parameter operations including brightness, contrast, saturation, sharpness. Set 3 is planned over the full operation list with the operation order fixed. Set 4 is planned over full operations with epsilon-greedy search. Set 5 is planned over the full operation list. Inputs represent the input image.}
  \label{tab:dif_op_list}
  \vspace{-4mm}
\end{table}

\begin{figure}[!tp] 
    \centering\includegraphics[width=1\columnwidth]{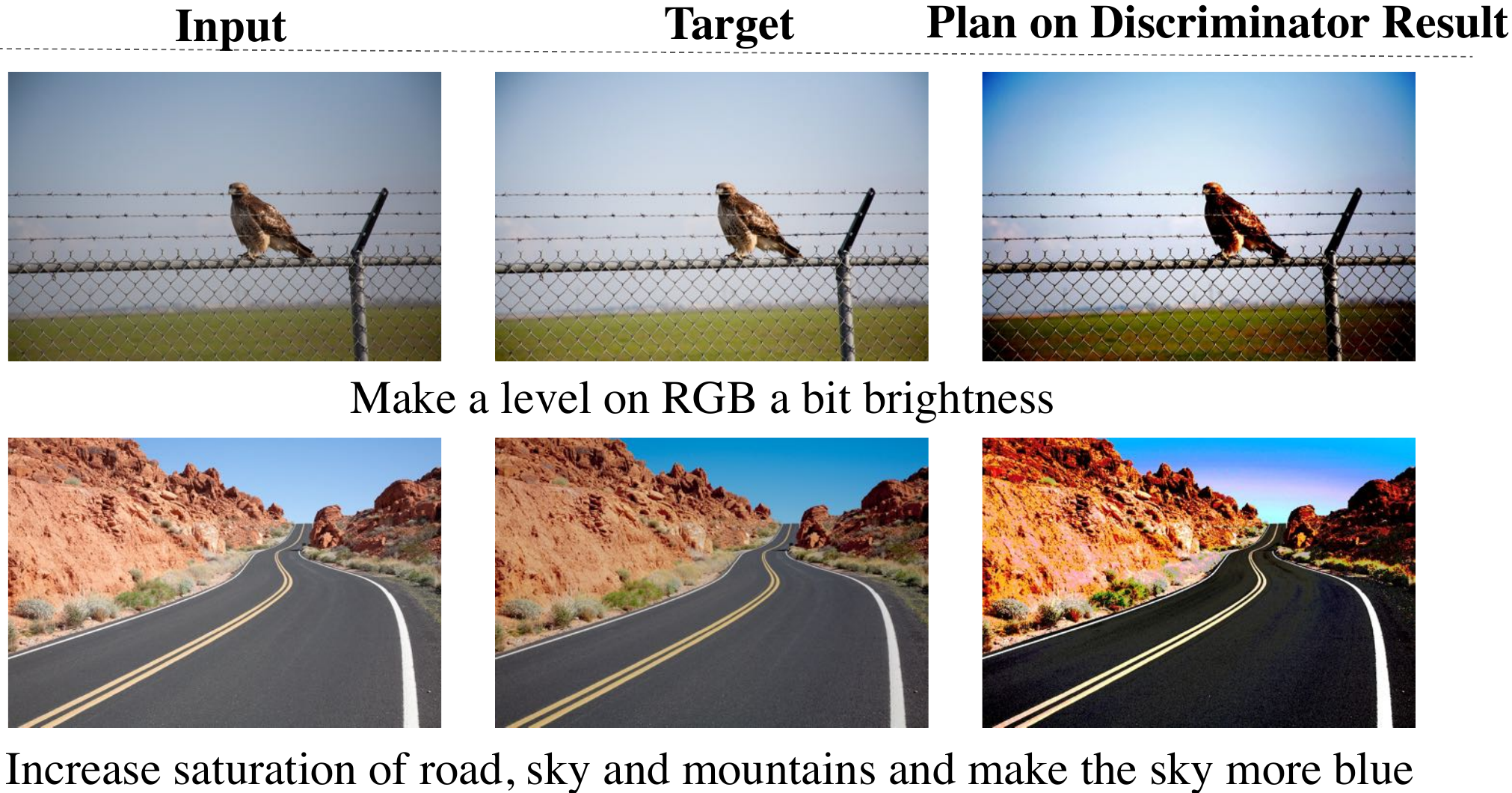} 
    \caption{Planning through a discriminator.} 
    \label{fig:plan_disc}
    \vspace{-3mm}
\end{figure} 

\subsection{Ablation Study}
\label{sec:ablation}
Due to space limit, the ablation study of different network structures is moved to Appx.~\ref{appx:historical_effect} and the investigation of alternative image loss is in Appx.~\ref{appx:diff_img_loss}.\\
\noindent\textbf{Trade-off between L1 and variance}.
We sample operation parameter $\alpha_t$ from $\mcal{N}(\alpha_t;\mu_{\alpha_t}, \sigma_\alpha)$ while training the L1 loss.
We set $\sigma_\alpha=Rh/3$, where $R$ is the half range of the parameter, $h$ is the gaussian width controller.
Interestingly, the L1 and variance of T2ONet can be traded-off by adjusting $\sigma_\alpha$. 
Fig.~\ref{fig:tradeoff} manifests that the image variance can be enlarged by increasing $h$, but in turn, resulting in higher L1. The detailed result table is in Appx.~\ref{appx:trade_l1_var}.
Moreover, Fig.~\ref{fig:vis_variance} shows that while all of the models are sensitive to requests, the model trained with larger $h$ produces more diversified results.

\noindent\textbf{Planning with different operation lists, operation orders and planning methods.}
According to both the planning and T2ONet editing performance in Tab.~\ref{tab:dif_op_list}, set 1, 2, 5  shows that the performance substantially increases as the operation candidate list becomes larger.
Planning with different single operation and different max step $N$ is studied in Appx.~\ref{appx:single_list}.
Set 3 and 5 compare the difference between fixed and our searched operation order. 
It shows the searched order is slightly better than the fixed one for planning (might because the improvement space for planning is limited), but it will bring a larger improvement for T2ONet. 
Set 4 and 5 indicate that the original version is better than alternative $\epsilon$-greedy policy~\cite{sutton2018reinforcement}, detailed in Appx.~\ref{appx:other_plan}.

% ------------- Figure of extension of operation planning --------------

\begin{figure}[!tp] 
    \centering\includegraphics[width=0.9\columnwidth]{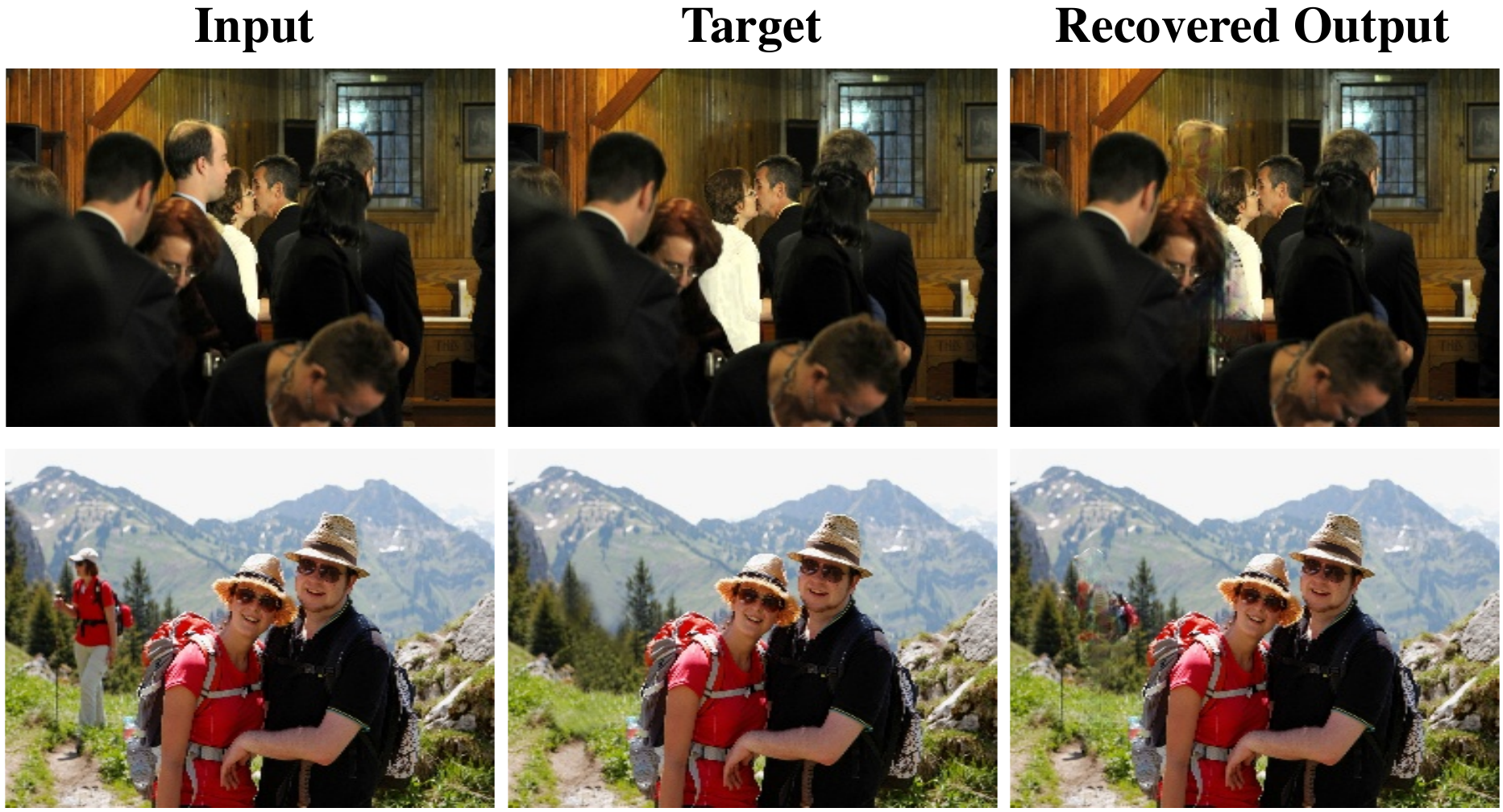} 
    \caption{Planning on local editing.} 
  \label{fig:plan_local}
    \vspace{-4mm}
\end{figure}

% \noindent\textbf{Different operation order}.
% Our operation planning algorithm takes the spirit of greedy best-search~\cite{russell2016artificial} and thus decides the operations order in a best-first fashion. Nonetheless, since the parameter is optimally selected for each operation, we doubt whether the order of operation still matters in the planning process. 
% Set 3 and 5 in Tab.~\ref{tab:dif_op_list} compare the difference of fixed and best-first operation order. 
% It shows the best-first order is slightly better than fixed one for planning, but it will bring a larger improvement for T2ONet. 
% Such phenomenon is possible because the improvement space for planning result is limited, so the gain is slighter for planning than inference. \xn{can move to supp}

\subsection{Extensions of Planning Algorithm}
\label{sec:discussion}
%Here, we discuss the potential of our operation planning and future directions.
\noindent\textbf{Planning through a discriminator.}
% \Jing{Then, talk about how to use the sequence to supervise, and the intuition to add the L1 loss at the final. And also say the discriminator is also compared into it in ablation study.}
% \Jing{1. The way to use the discriminator. How to train it, and how the loss would help the seq2seq model.
% Also, to deal with high resolution problem, use the discriminator from pix2pixHD, so to illustrate that the discriminator have multilevel}
% \Jing{2. The novelty about directly optimize the discriminator, such that we don't use generator.
% Talks some advantages of this. 1) check if the discriminator is good. 2) directly optimize the discriminator can achieve better score then GAN (compare it with experiment), so hopefully we can get better result. 3) From the experiment result, the discriminator have trouble being sensitive to different request, but just improve the overal quality of the image. But the adversarial loss is hard to apply because it is hard to select the disjoint sentences from the out dataset. So this should be left as future work}
We leverage a discriminator $D$ that takes as input a pair of images and a request and outputs a score indicating the editing quality. Such $D$ is pretrained with adversarial loss on T2ONet (see Appx.~\ref{appx:diff_img_loss} for detail). We define the new cost function as $\mrm{cost}(I)=1 - D(I_0, I, Q)$, and apply it to Alg.~\ref{alg:forwardSearch}. 
Interestingly, such planning can still produce some visually pleasing results, shown in Fig.~\ref{fig:plan_disc}.
Although its quantitative results are worse than our default training performance, using a pretrained image-quality discriminator to edit an image brings a new perspective for image editing. Another advantage is its flexibility such that the same discriminator can be applied on a different set of operations while previous methods require retraining.
% However, since the discriminator only sees the negative sample generated by T2ONet, the ability of the discriminator is not strong enough to be directly planned on.
% Nevertheless, this could be a promising future direction because it could also be a visualization to interpret or diagnose discriminator.

% \begin{figure}[!tp] 
%   \centering\includegraphics[width=1\columnwidth]{pics/vis_plan_disc.pdf} 
%   \caption{Left: visualization on planning over discriminator. Right: structure of the discriminator.} \label{fig:plan_disc}
% \end{figure} 

% \begin{figure}[!tp] 
%     \centering\includegraphics[width=\columnwidth]{pics/local_plan.pdf} 
%     \caption{Visualization for planning on local editing.} 
%   \label{fig:plan_local}
% \end{figure} 

\noindent\textbf{Planning for local edit.}
Our operation planning can generalize to local editing (\eg ``remove the man in the red shirt on the left"). Given the input and target image, we can use the pretrained panoptic segmentation network~\cite{xiong2019upsnet} to get a set of segments in the input image. With our planning algorithm (adding a new loop for segments, adding inpainting as one operation), we can get the pseudo ground truth, including the inpainting operation and its edited area, which can train a local editing network like~\cite{shi2020benchmark}.
Its full algorithm is described in the Appx.~\ref{appx:local_edit}.

% Our operation planning can generalize to local editing (\eg). Given a set of $K$ zero-one image mask $M$, 
% we constraint the editing operation only be applied to the masked area. Therefore, we add an inner loop over all $K$ mask candidates, resulting in time complexity goes to $O(NB|\mathcal{O}|K)$.
% However, $K$ can be removed if we know the grounded mask for each operation.
% Fig.~\ref{fig:plan_local} shows the planning result for local object removing. The operation planning can correctly find the object mask from a set of masks candidates computed by UPSNet~\cite{xiong2019upsnet}.

% \noindent\textbf{Resolution Invariant}
% One advantage of our mothod is resolution invariant. Figure~\ref{fig:res_invariant} shows our method could be applied to high-resolution image while the GAN-based method can not.
\section{Conclusion}
We present an operation planning algorithm to reverse-engineer the editing through input image and target image, and can even generalize to local editing.
A Text-to-Operation editing model supervised by the pseudo operation sequence is proposed to achieve a language-driven image editing task.
We proved the equivalence of the image loss and the deterministic policy gradient.
Comparison experiments manifest our method is superior to other GAN-based and RL counterparts on both MA5k-Req and GIER Images.
The ablation study further investigates the trade-off between L1 and request-sensitivity and analyzes the factors that affect operation planning performance.
Finally, we extend the operation planning to a discriminator-based planning and local edit.

\noindent\textbf{Acknowledgments}\quad This work was supported in part by an Adobe research gift, and NSF 1813709, 1741472 and 1909912. The article solely reflects the opinions and conclusions of its authors but not the funding agents.

\clearpage

\appendix

% \begin{table}[]
%   \centering
%   \caption{L1 distance to target image over different maximum steps and single operations on MIT-Adobe 5k dataset. Inputs represent input image.}
%   \ra{1.2}
%   \scalebox{0.7}{
%   \begin{tabular}{@{}ccccccccccccccc@{}}
%   \toprule
%   & \multicolumn{6}{c}{Maximum step} & \phantom{ab} & \multicolumn{6}{c}{Single Operation} & input\\
%   \cmidrule{2-7} \cmidrule{9-14} 
%   & 1 & 2 & 3 & 4 & 5 & 6 & & brightness & contrast & saturation & sharpness & tone & color & \\
%   \midrule
%   planning (train) & 0.0256 & 0.0145 & 0.0139 & 0.0137 & 0.0136 & 0.0136 && 0.0521 & 0.0859 & 0.1037 &0.1163 & 0.0277 & 0.0260 & 0.1202\\  % 0.0173
%   T2ONet (test) &  &  &  &  &  &  && 0.1315 & & & & & 0.1129 & 0.1190 \\
%   \bottomrule
%   \end{tabular}}
%   \label{tab:abla}
%   \vspace{-3mm}
% \end{table}

\noindent {\Large  \textbf{Appendix}}
\section{Ablation Study}

\subsection{Different Image Loss} \label{appx:diff_img_loss}
Inspired by the GAN-based image-to-image translation~\cite{wang2018high}, we also try to apply adversarial loss $\mcal{L}_{adv}$ using discriminator $D$, whose structure is shown in Fig.~\ref{fig:disc_structure}. 
The adversarial loss is expressed as:
\begin{align}
  \mathcal{L}_{adv} =& -\mathbb{E}_{(I_0, I_g)}[\log(D(I_0, I_g, Q))] \nonumber\\
  & - \mathbb{E}_{(I_0, I_T)}[\log(1 - D((I_0, I_T, Q)))]. \label{eqn:discriminator}
\end{align}
Denote the whole parameter for the T2ONet as $\Theta_G$ and discriminator as $\Theta_D$,  the objective for adversarial loss is $\min_{\Theta_G}(\max_{\Theta_D}\mcal L_{adv})$.
The effect of L1 and adversarial loss is shown in Tab.~\ref{tab:seq2seq}.
We observe that adding the image level loss can significantly improve T2ONet, because the operation supervision is trained in teacher forcing fashion, which easily accumulates error at each step.
A supervision at the final image help correct the error at the final image.
And without image supervision, the variance drops significantly, indicating the model has a very similar output for different requests.
Moreover, the L1 loss is better than the adversarial loss. It might because adversarial loss is good at generating sharper and more detailed images~\cite{goodfellow2016nips,ledig2017photo}, but our operation will not reduce the detail/texture of the image, so the adversarial loss may not help as much as L1 loss, which pushes the similarity of the generated image to target image in a more direct way.
And the combination of L1 and adversarial loss is still weaker than solely L1 loss in general, probably because we directly use $\mcal L = \mcal L_{L1} + \mcal{L}_{adv}$ and didn't fine-tune the balance weight.
Hence, to facilitate our model design, we purely use L1 loss as the image loss.
The visual comparison of different final image losses is shown in Fig.~\ref{fig:vis_ablation} and we find that without L1 loss or changing L1 to adversarial loss, the visual appearance is less similar to target and less appearing. 

\subsection{Trade-off between L1 and Variance} \label{appx:trade_l1_var}
Tab.~\ref{tab:h_tradeoff} shows the complete evaluation for the trade-off between L1 and variance.

\subsection{Effect of historical operations, images and attention for T2ONet} \label{appx:historical_effect}
Our standard T2O Cell takes in the previous operation and image. The comparison with only either of them is shown in Tab.~\ref{tab:seq2seq}, indicating that just image or operation performs no better than their combination. 
One exception is the variance for the only operation is better than combined, which means without the historical image as the feedback, the editing will be less controlled and be more diversified.
Also, the attention mechanism help improve the performance according to Tab.~\ref{tab:seq2seq}.

\begin{figure}[!tp] 
  \centering\includegraphics[width=1\columnwidth]{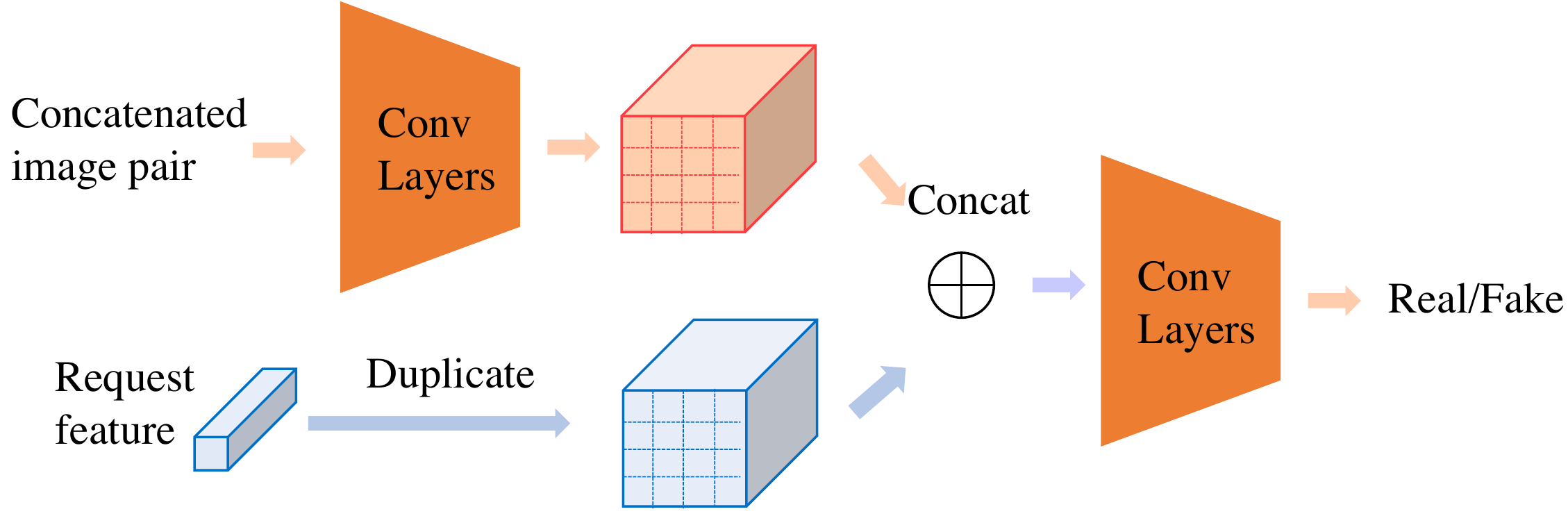} 
  \caption{Structure of the discriminator used for adversarial loss.} 
  \label{fig:disc_structure}
\end{figure} 
%-------------- comparison with different image loss -------------
\begin{table}[t]
  \centering
  \ra{1.2}
  \scalebox{0.9}{
  \begin{tabular}{ccrrrr}
  \toprule
   L1 & Adv & L1$\downarrow$ & SSIM $\uparrow$ & FID $\downarrow$ & $\sigma_{\times 10^2}$$\uparrow$ \\
  \midrule
  \xmark & \xmark & 0.0949 & 0.8300 & 8.2482 & 0.0532 \\
  \cmark & \xmark & \textbf{0.0784} & 0.8459 & \textbf{6.7571} & \textbf{0.7190} \\
  \xmark & \cmark & 0.0901 & 0.8031 & 9.4600 & 0.5825\\
  \cmark & \cmark & 0.0801 & \textbf{0.8464} & 6.9436 & 0.5671 \\
  \bottomrule
  \end{tabular}}
  \caption{Ablation study of different losses and network structures on the MA5k-Req test set. L1, Adv represent L1 and adversarial loss, respectively.}
  \label{tab:seq2seq}
\end{table}

% ------------  trade-off for sample width ------------.
\begin{table}
  \centering
  \begin{tabular}{crrrr}
  \toprule
   $h$ & L1$\downarrow$ & SSIM $\uparrow$ & FID $\downarrow$ & $\sigma_{\times 10^2}$$\uparrow$ \\
  \midrule
  0 &  \textbf{0.0784} & \textbf{0.8459} & \textbf{6.7571} & 0.7190 \\
  0.01& 0.0809 & 0.8487 & 7.2789 & 1.1008 \\
  0.1 & 0.0979 & 0.8090 & 8.8763 & \textbf{2.1482}\\
  \bottomrule
  \end{tabular}
  \caption{L1 and variance trade-off by training with different parameter sampling variance (reflected by $h$) on the MA5k-Req test set.}
  \label{tab:h_tradeoff}
\end{table}

%--ablation study of history operation and attention---
\begin{table}[t]
  \centering
  \ra{1.2}
  \scalebox{0.9}{
  \begin{tabular}{crrrr}
  \toprule
    & L1$\downarrow$ & SSIM $\uparrow$ & FID $\downarrow$ & $\sigma_{\times 10^2}$$\uparrow$ \\
  \midrule
  w./o. image & 0.0863 & 0.8332 & 7.7869 & \textbf{1.1950}\\
  w./o. operation & 0.0837 & 0.8424 & 7.6559 & 0.3257\\
  w./o. attention & 0.1088 & 0.8087  & 8.4587 & 0.8872\\
  full model & \textbf{0.0784} & \textbf{0.8459} & \textbf{6.7571} & 0.7190\\
  \bottomrule
  \end{tabular}}
  \caption{W./o. image, operation, and attention indicate the T2OCell without using in the intermediate image, operation, and attention on MA5k-Req test set.}
  \label{tab:seq2seq}
\end{table}

\subsection{Comparison of other possible planning method} \label{appx:other_plan}
Since our operation planning is based on greedy best-search, such greedy method does not guarantee an optimal solution. 
Inspired by the $\epsilon$-greedy policy~\cite{sutton2018reinforcement} applied in RL, we further compare a variant called $\epsilon$-greedy operation planning to incorporate randomness to further approach optimal.
The only difference is that there is $5\%$ possibility that the operation is randomly selected, rather than the top choice.

% -------------Table for different single operation --------------------
\begin{table*}[]
  \centering
  \ra{1.2}
  \begin{tabular}{@{}ccccccccc@{}}
  \toprule
  Operation & Brightness & Contrast & Saturation & Sharpness & Tone & Color &\phantom{ab}& Input \\
  \midrule
  planning (train) & 0.0521 & 0.0859 & 0.1037 &0.1163 & \textbf{0.0277} & \textbf{0.0260} && 0.1202\\
  T2ONet (test) & 0.1315 & 0.1178 & 0.1163 & 0.1256 & \textbf{0.1006} & \textbf{0.1129} && 0.1190 \\
  \bottomrule
  \end{tabular}
  \caption{L1 distance to target image over different \textbf{single operations} on MA5k-Req dataset. Input represents the distance of the input image to the target image. Planning results are on the training set, and T2ONet results are on the testing set.}
  \label{tab:abla_single_operation}
\end{table*}

\begin{table*}[]
  \centering
  \ra{1.2}
  \begin{tabular}{@{}ccccccccc@{}}
  \toprule
  Max Step & 1 & 2 & 3 & 4 & 5 & 6 & \phantom{ab}&Input \\
  \midrule
  Planning (train) & 0.0256 & 0.0145 & 0.0139 & 0.0137 & 0.0136 & 0.0136 && 0.1202 \\ 
  \bottomrule
  \end{tabular}
  \caption{L1 distance to the target image with different \textbf{maximum editing steps} on MA5k-Req dataset. Input represents the distance of the input image to the target image.}
  \label{tab:abla_max_step}
  \vspace{-3mm}
\end{table*}

\subsection{Effect of different single operation lists and different maximum operation steps}\label{appx:single_list}
We further study the comparison of only applying single operation. 
Table~\ref{tab:abla_single_operation} presents the editing results of planning and T2ONet using only single operation.
The most effective operations are ``tone'' and ``color'', because they have 8 and 24 parameters, respectively, and thus have stronger editing ability than single-parameter operations . The results also suggest that the single-step editing results are worse than our proposed multi-step editing results. 
Also, we track the effect of different maximum operation steps, the planning results are shown in Tab.~\ref{tab:abla_max_step}.
From the perspective of planning, maximum steps 4, 5, 6 do not have much difference. This suggests that we could reduce the editing step or find the best trade-off between the editing effect and editing time complexity in the future.

\section{Reinforcement Learning}
\subsection{Details of RL baseline} \label{appx:rl_baseline}

Now we reformulate the editing problem into a partial observed Markov decision process and introduce an RL baseline. 
Following the symbol notations and the problem formulation in the main paper that the editing process is a sequential action decision problem, we augmented with reward $r_{t+1}$ for action $a_t$, the problem can be reformulated into a Markov decision process and solved by RL.
Following ~\cite{hu2018exposure}, the reward is set to indicate the incremental image quality, which is adapted as the reduction of the image cost
\begin{equation} \label{eqn:reward}
  r_t = \mrm{cost}(I_{t-1}) - \mrm{cost}(I_{t}),
\end{equation}
where $\mrm{cost}(I)$ can be any image loss and is set as $||I-I_g||_1$ in our experiment. 
Since the reward for the ``END'' action is hard to design (the reward in Eq.~(\ref{eqn:reward}) is zero for ``END'' action), we set every episode fixed $T$ steps ($T=5$ as \cite{hu2018exposure}).
The actions are sampled from policy $\pi=(\pi_o, \pi_\alpha)$, where $\pi_o=P(o|s), \pi_\alpha=P(\alpha|o,s)$, leading to the trajectory $\Pi=\{s_0, a_0, s_1, r_1, ..., s_T, r_T\}$. 
The $P(o|s), P(\alpha|o,s)$ is computed the same way as T2ONet.
With the accumulated reward defined as $G_t = \sum_{\tau=0}^{T-t}\gamma^{\tau}r_{t+\tau}$ ($\gamma=1$ as \cite{hu2018exposure}), the goal is to optimize the objective $J(\pi)=\mbb E_{(I_0, Q)\sim P(\mcal D), \Pi\sim \pi} G_1$, where $p(\mcal{D})$ is the distribution of the dataset.
Denoting $\theta_o$ and $\theta_\alpha$ as the respective model parameter involving in the computation of $o$ and $\alpha$, the discrete policy $\pi_o$ is optimized via REINFORCE~\cite{williams1992simple}:
\begin{equation}
  \nabla_{\theta_o}J(\pi)=\underset{\substack{(I_0, Q)\sim P(\mcal D)\\ o_t\sim \pi_o, \alpha_t\sim \pi_\alpha}}{\mbb E} \sum_{t=0}^{T-1} G_{t+1}\nabla_{\theta_o}\log\pi_o(o_t).
\end{equation}
For the continuous policy $\pi_\alpha$, we resort to DPG~\cite{silver2014deterministic}. Different from the common setting \cite{silver2014deterministic,hu2018exposure} where the Q function is approximated with a neural network to make it differentiable to action, we approximate $Q$ as $G$ since our $G_{t+1}$ is already differentiable to $\alpha_t$, resulting in the DPG as 
\begin{equation}
  \nabla_{\theta_\alpha}J(\pi) = \underset{\substack{(I_0, Q)\sim P(\mcal D)\\ \alpha_t\sim \pi_\alpha}}{\mbb E}
  \sum_{t=0}^{T-1}\nabla_{\alpha_t}G_{t+1} \nabla_{\theta_\alpha}\alpha_t. 
  \label{eqn:alpha_grad}
\end{equation} 

In short, the major difference of our RL optimization from \cite{hu2018exposure} is that we replace Q function approximated by neural network in \cite{hu2018exposure} with $G$ in both discrete and continuous policies, avoiding the complexity for training the Q network. 
The full algorithm for our RL baseline is in Appx.~\ref{sec:rl_alg}.

In our experiments, the sampling for $o$ is based on $\pi_o$ with $\epsilon$-greedy policy where the $\epsilon=0.05$.
The sampling for $\alpha$ is based on $\pi_\alpha$ where the gaussian width controller $h=0.1$. 
The other implementation details are the same with our main experiments.

\subsection{Equivalence of image loss and DPG}\label{appx:eq_proof}
Now, we show the equivalence between image loss and DPG using the following theorem:
\begin{theorem}
  The DPG for $\alpha$ in Eq.~(\ref{eqn:alpha_grad}) can be rewrite as 
\begin{equation}
  \nabla_{\theta_\alpha}J(\pi) = -\underset{\substack{(I_0, Q)\sim P(\mcal D)\\ \alpha\sim \pi_\alpha}}{\mbb E}
 \frac{\p\mrm{cost}(I_T)}{\p \theta_\alpha}.\label{eqn:equality}
\end{equation}
\end{theorem}
\begin{proof}
Substituting Eq.~(\ref{eqn:reward}) and $\gamma=1$, $G_t$ can be simplified as 
\begin{align}
  G_t &= \sum_{\tau=0}^{T-t} \big (\mrm{cost}(I_{t+\tau-1}) - \mrm{cost}(I_{t+\tau})\big ) \nonumber\\
   &= \mrm{cost}(I_{t-1}) - \mrm{cost}(I_{T}).
\end{align}
Since $I_t$ is independent of $\alpha_t$, we have
\begin{equation}
\nabla_{\alpha_t} G_{t+1}= \frac{\p (\mrm{cost}(I_{t}) - \mrm{cost}(I_{T}))}{\p \alpha_t}=-\frac{\p \mrm{cost}(I_T)}{\p \alpha_t}.
\end{equation}
Therefore, the summation in Eq.~(\ref{eqn:alpha_grad}) can be expressed as
\begin{equation}
  \sum_{t=0}^{T-1} \nabla_{\alpha_t}G_{t+1} \nabla_{\theta_\alpha}\alpha_t = -\sum_{t=0}^T\frac{\p\mrm{cost}(I_T)}{\p \alpha_t}\frac{\p \alpha_t}{\p \theta_\alpha} = -\frac{\p \mrm{cost}(I_T)}{\p \theta_\alpha}
  \label{eqn:sum_equality}
\end{equation}
According to Eq.~(\ref{eqn:sum_equality}), Eq.~(\ref{eqn:alpha_grad}) is equivalent to Eq.~(\ref{eqn:equality}).
\end{proof}

\subsection{Algorithm for RL baseline}
\label{sec:rl_alg}
The full algorithm for our RL baseline is shown in Alg.~\ref{alg:rl}.
\begin{algorithm}[]
  \label{alg:rl}
\SetAlgoLined
\LinesNumbered
\KwIn{Training dataset $\mcal D$; learning rate $\beta$; max operation step $N=5$}
% \KwOut{Model $\theta=(\theta_o, \theta_\alpha)$}
\For{episode in $1:M$}{
  Sample $I_0, Q, I_g$ from $\mcal D$;\\
  Sample one editing episode from $\pi_o, \pi_\alpha$:\\
  $\{ I_0, a_0, I_1, r_1, a_2, I_2, \ldots, I_T, r_T \}$;\\
  $\Delta_{\theta_o}J= \sum_{t=0}^{T-1} G_{t+1}\nabla_{\theta_o}\log\pi_o(o_t)$;\\
  $\theta_o \leftarrow \theta_o + \beta\Delta_{\theta_o}J$;\\
  $\Delta_{\theta_\alpha}J = -\frac{\p \mrm{cost}(I_T)}{\p \theta_\alpha}$;\hfill \textcolor{ForestGreen}{\# $\mrm{cost}(I) = ||I - I_g||_1$}\\
  $\theta_\alpha \leftarrow \theta_\alpha + \beta\Delta_{\theta_\alpha}J$;\\
}
\Return $(\theta_o, \theta_\alpha)$
 \caption{RL}
\end{algorithm}

\subsection{Can operation planning benefit RL?} \label{appx:plan_help_rl}
Since the success of RL relies on the exploration of the action space, can the action sequence obtained from the operation planning algorithm help RL to better explore the action space, especially the continuous action?
To answer this question, similar to \cite{yi2018neural}, we firstly pretrain the model with the planned operations as supervision (same as T2ONet training loss), then finetune it using RL with only the target image supervision.
The result in Tab.~\ref{tab:rl_finetune} show that the pretraining does not help RL much on MA5k-Req, but significantly benefit RL on GIER.
As GIER has smaller size and more complex editing than FiveK, RL is struggling with the exploration of $\alpha$. The pretrained model can initialize a good exploration and thus the RL can work on GIER. 

% ------------  trade-off for sample width ------------.
\begin{table}
  \centering
  \scalebox{0.85}{
  \begin{tabular}{ccrrrr}
  \toprule
  Dataset & Pretrain & L1$\downarrow$ & SSIM $\uparrow$ & FID $\downarrow$ & $\sigma_{\times 10^2}$$\uparrow$ \\
  \midrule
  MA5k-Req & \xmark & 0.1007 & 0.8283 & 7.4896 & 1.6175 \\
  MA5k-Req & \cmark & 0.0955 & 0.8330 & 7.1413 & 1.4672 \\
  GIER     & \xmark & 0.2286 & 0.3832 & 132.1785 & 0.3978 \\
  GIER     & \cmark & 0.1052 & 0.8075 & 49.4183 & 1.0949\\
  \bottomrule
  \end{tabular}}
  \caption{The RL performance with and without operation-supervised pretrain on two datasets.}
  \label{tab:rl_finetune}
\end{table}

\section{Planning for Local Editing}\label{appx:local_edit}
Our operation planning can generalize to local editing. Given a zero-one image mask $M$, we redesign the image editing function as $I_\textrm{out} = o(I, \alpha)\odot M + I\odot(1-M)$, where $\odot$ is element-wise product; thus only the masked part is edited. 
Given $K$ mask candidates, we can add an inner loop over all $K$ mask candidates to further generate $K$ edited images each time.
In this case, the time complexity goes to $O(NB|\mathcal{O}|K)$.
However, $K$ can be removed if we know the grounded mask for each operation.
Its full algorithm is described in the Alg.~\ref{alg:operation_planning_local}.
We use UPSNet~\cite{xiong2019upsnet} to obtain the mask candidates and use~\cite{nazeri2019edgeconnect} for removing/inpainting operation.
Given each operation with its region, it could also train our T2ONet augmented with the grounding model.
Since this paper focuses on global operation planning, it will be left for future work.
% operation planning
\begin{algorithm}[]
  \label{alg:operation_planning_local}
\SetAlgoLined
\LinesNumbered
\KwIn{$I_0$, $I_{g}$, max operation step $N$, threshold $\epsilon$, beam size $B$, operation set $\mathcal{O}$, mask set $\mathcal{M}$}
$p$=$[I_0]$\\
$cost(I) = ||I - I_g||_1$\\
\For{$t$ in $1:N$}{
  $q \leftarrow [\ ]$\\
  \For{$I \in p$}{
    \For{$o \in \mathcal{O}$}{
      \For{$M \in \mathcal{M}$}{
        $\alpha^* = \arg\min_\alpha cost(o(I, \alpha)\odot M + I\odot(1-M))$\\
        $I^* \leftarrow o(I, \alpha^*)\odot M + I\odot(1-M)$\\
        $q \leftarrow q\cup I^*$\\
      }
    }
  }
  $q\leftarrow Sort(q), sort key=cost(I^*)$ \\
  $p = q[:B]$\\
  \For{$I \in p$}{
    \If{$cost(I) < \epsilon$}{
      Break All Loop
    }
  }
}
$\{o_t\}, \{\alpha_t\}, \{M_t\}, \{I_t\}\leftarrow Backtracking(p)$

\Return $\{o_t\}, \{\alpha_t\}, \{M_t\}, \{I_t\}$
 \caption{Operation Planning with Local Editing}
\end{algorithm}

%------------------
\begin{figure*}[!tp] 
  \centering\includegraphics[width=2\columnwidth]{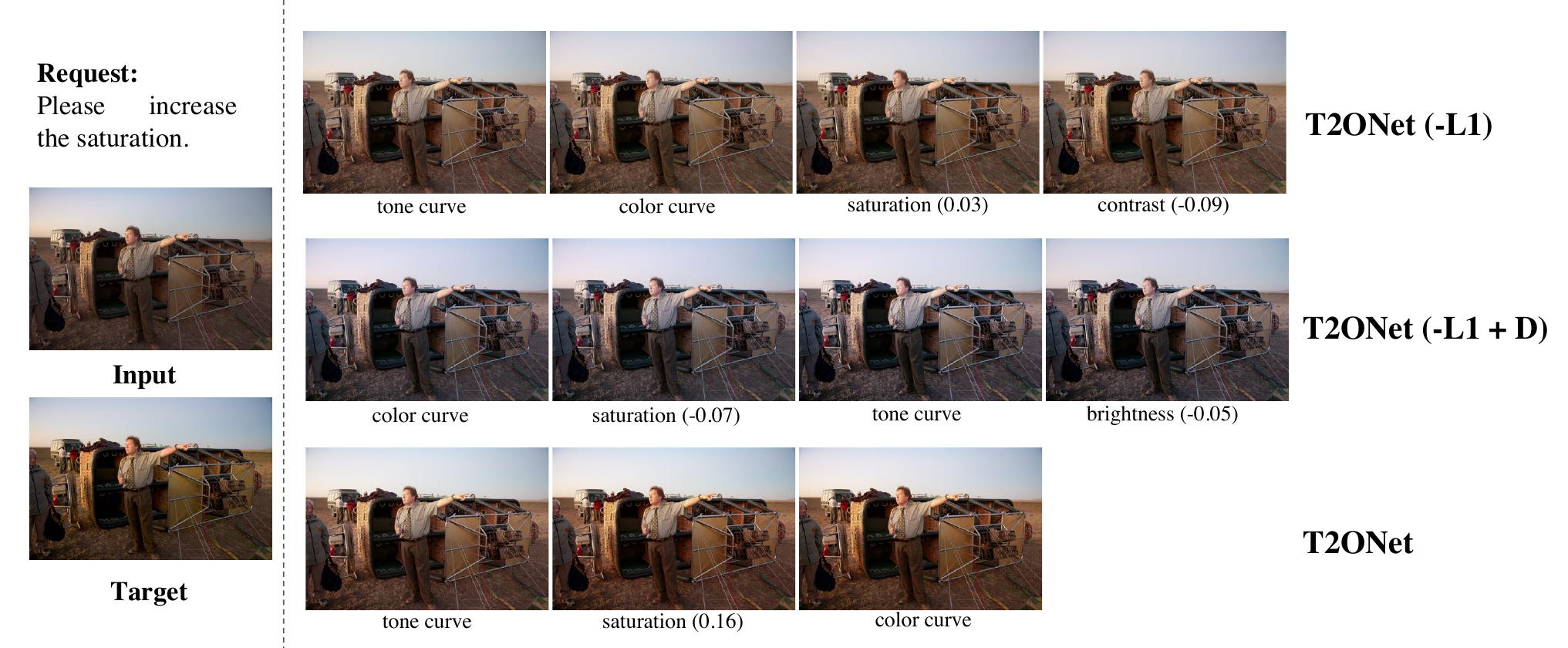} 
  \caption{Visualization for ablation study methods. T2ONet(-L1) is the modified version without L1 loss, T2ONet(-L1+D) is to replace the L1 loss to adversarial loss.} 
  \label{fig:vis_ablation}
\end{figure*} 
%------------------

% -----------------
\begin{figure}[!tp] 
  \centering\includegraphics[width=\columnwidth]{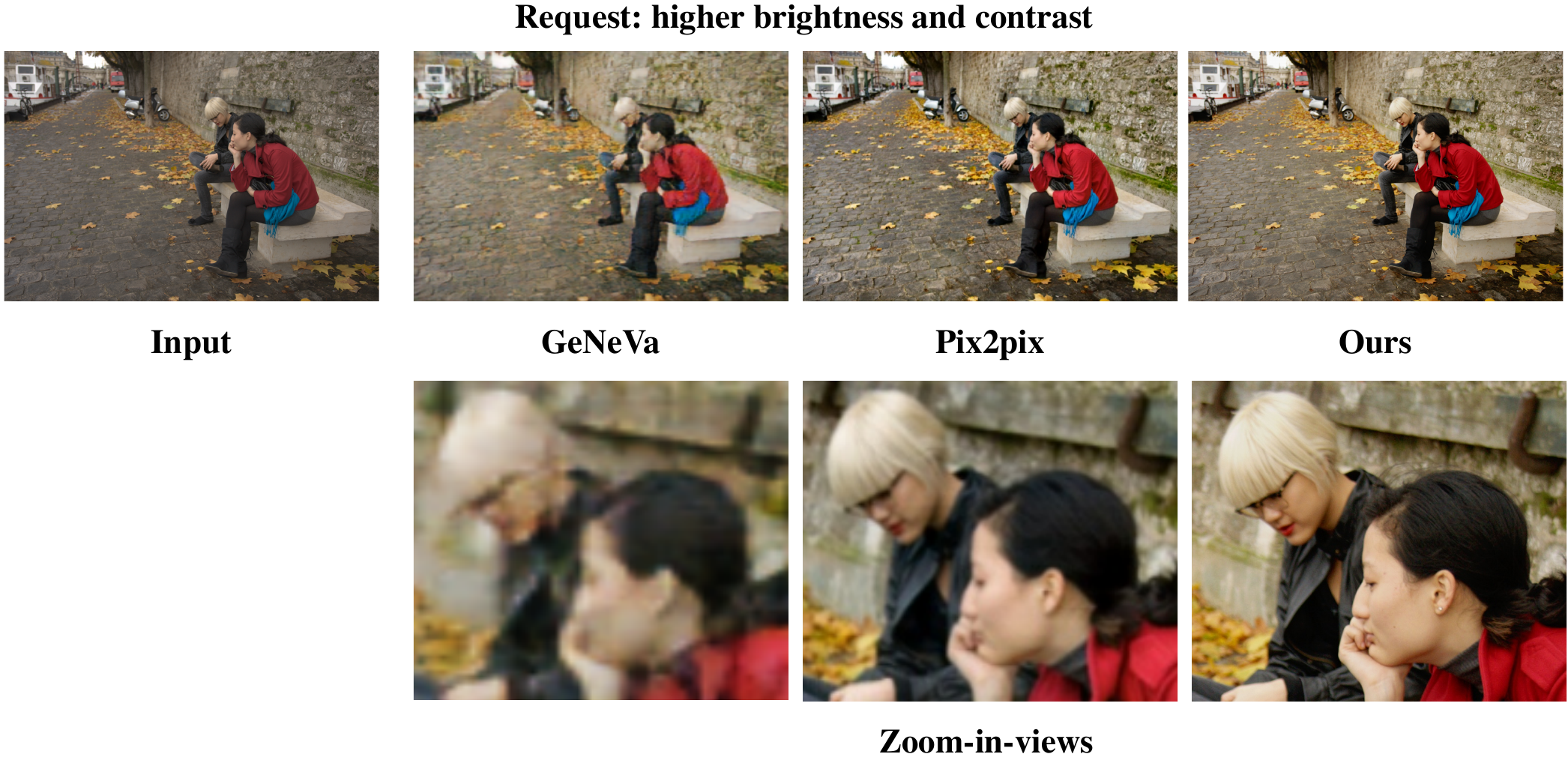} 
  \caption{Compared with the GAN-based method GeNeVa and Pix2pixAug, although all the methods conduct the correct editing, our method has no pixel distortion and is independent to image resolution.} 
  \label{fig:vis_hr}
\end{figure} 

% -----------------
\begin{figure}[!tp] 
  \centering\includegraphics[width=\columnwidth]{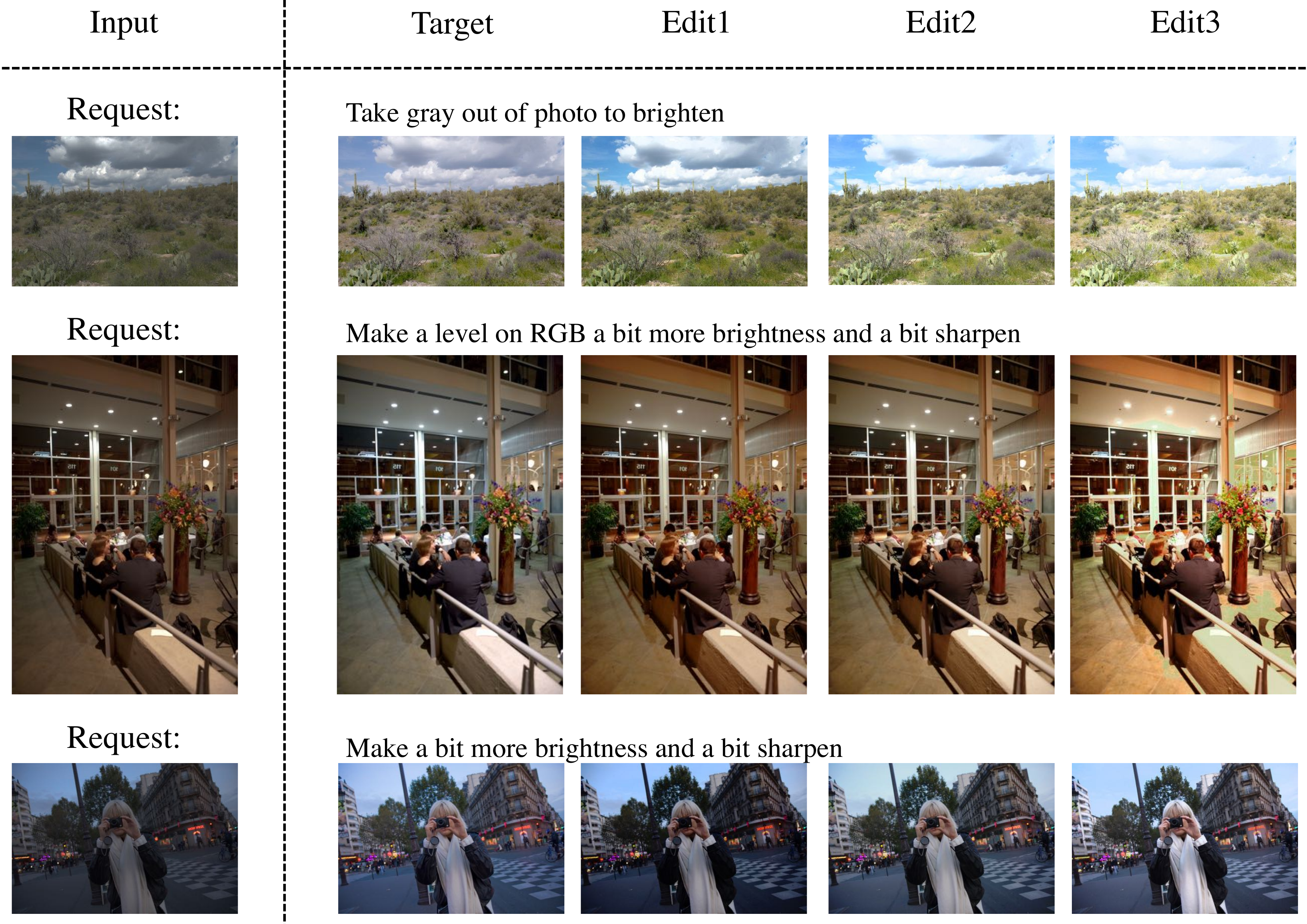} 
  \caption{Visualization for diversified output given the same input and request by sampling the operation parameter at inference stage.} 
  \label{fig:vis_diverse}
\end{figure}

\begin{table}[]
  \centering
  \ra{1.2}
  \begin{tabular}{@{}lrrr@{}}
  \toprule
  Method & Planning (s) & Train (s) & Test (s)\\ 
  \midrule
  Pix2pixAug~\cite{wang2018learning} & 18.58 & 0.37 & 0.05 \\ 
  T2ONet & - & 1.16 & 0.16 \\ 
  \bottomrule
  \end{tabular}
  \caption{The average running time comparison for GAN-based method Pix2pixAug~\cite{wang2018learning} and our method. The training time is computed in batch size 64, and test and planning time are computed in batch size 1.}
  \label{tab:time_comparison}
  \vspace{-3mm}
\end{table}

\section{Time Analysis}
We compare the running time of T2ONet and Pix2pixAug~\cite{wang2018learning} in Tab.~\ref{tab:time_comparison}.
For T2ONet, the computing-intensive planning is a pre-processing step and only needs to be computed once, and our model shows faster train and test speed than Pix2pixAug, indicating that our method not only has better editing quality, but also is computational cheaper.

\section{More advantages of T2ONet}
\subsection{Resolution independent editing} \label{appx:res_independ}
Our model will conduct resolution-independent editing and can produce the output with the same resolution as the input image. 
However, the GAN-based method suffers from generating high-resolution images, such comparison is shown in Fig.~\ref{fig:vis_hr}.

%---------------
\begin{figure*}[]
  \centering
  \includegraphics[width=2\columnwidth]{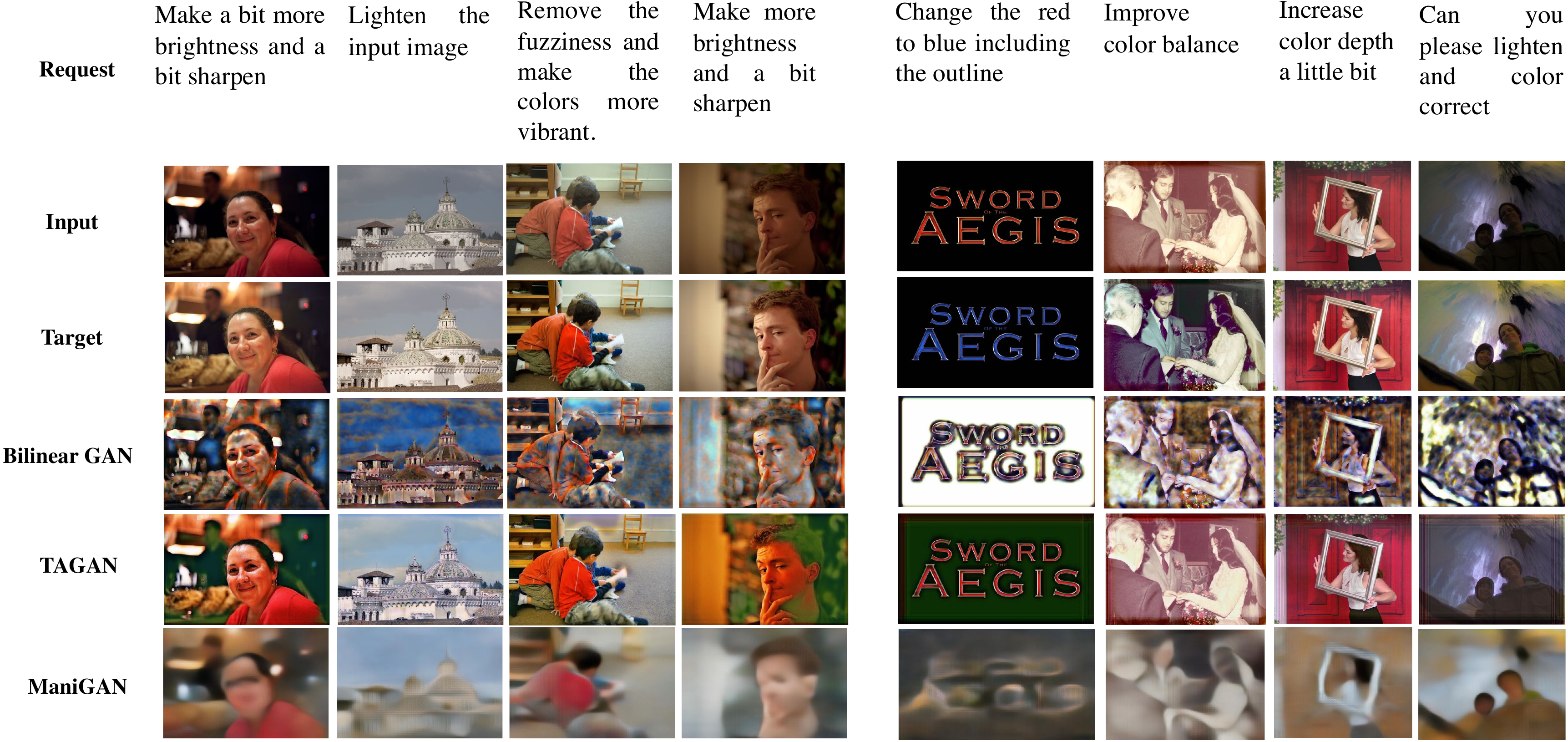}
  \caption{The comparison results for BilinearGAN, TAGAN and ManiGAN on MA5k-Req (left) and GIER (right) datasets} 
  \label{fig:vis_comp_exp_more}
\end{figure*}
%---------------

\subsection{Inference with multiple possible output}
We have discussed the trade-off between L1 and variance by sampling the operation parameter during training stage, and such variance is measured over the outputs edited from different requests, with the purpose of indicating the model's language-sensitivity
However, our model can even generate multiple output given the same request by sampling the operation parameter at the inference stage, whose result is shown in Fig.~\ref{fig:vis_diverse}.

\section{More visual results}
\subsection{Comparison Methods} \label{appx:vis_comp}
Here we show the comparison visual result of BilinearGAN, TAGAN and ManiGAN in Fig~\ref{fig:vis_comp_exp_more}. The visual results for ManiGAN is quite blur, and its L1, SSIM, FID are 0.1398, 0.5177, 157.4145 on MA5k-Req and 0.1834, 0.4938, 234.6784 on GIER, respectively. Therefore we did not do user study for this method.  

\subsection{T2ONet}
More visual results for T2ONet on MA5k-Req and GIER are shown in Fig.~\ref{fig:vis_res_FiveK} and Fig.~\ref{fig:vis_res_GIER}, respectively.

\subsection{Operation Planning}
More visualization of the operation editing process is shown in Fig.~\ref{fig:vis_plan_more}

%---------------
\begin{figure*}[]
  \centering
  \includegraphics[width=2\columnwidth]{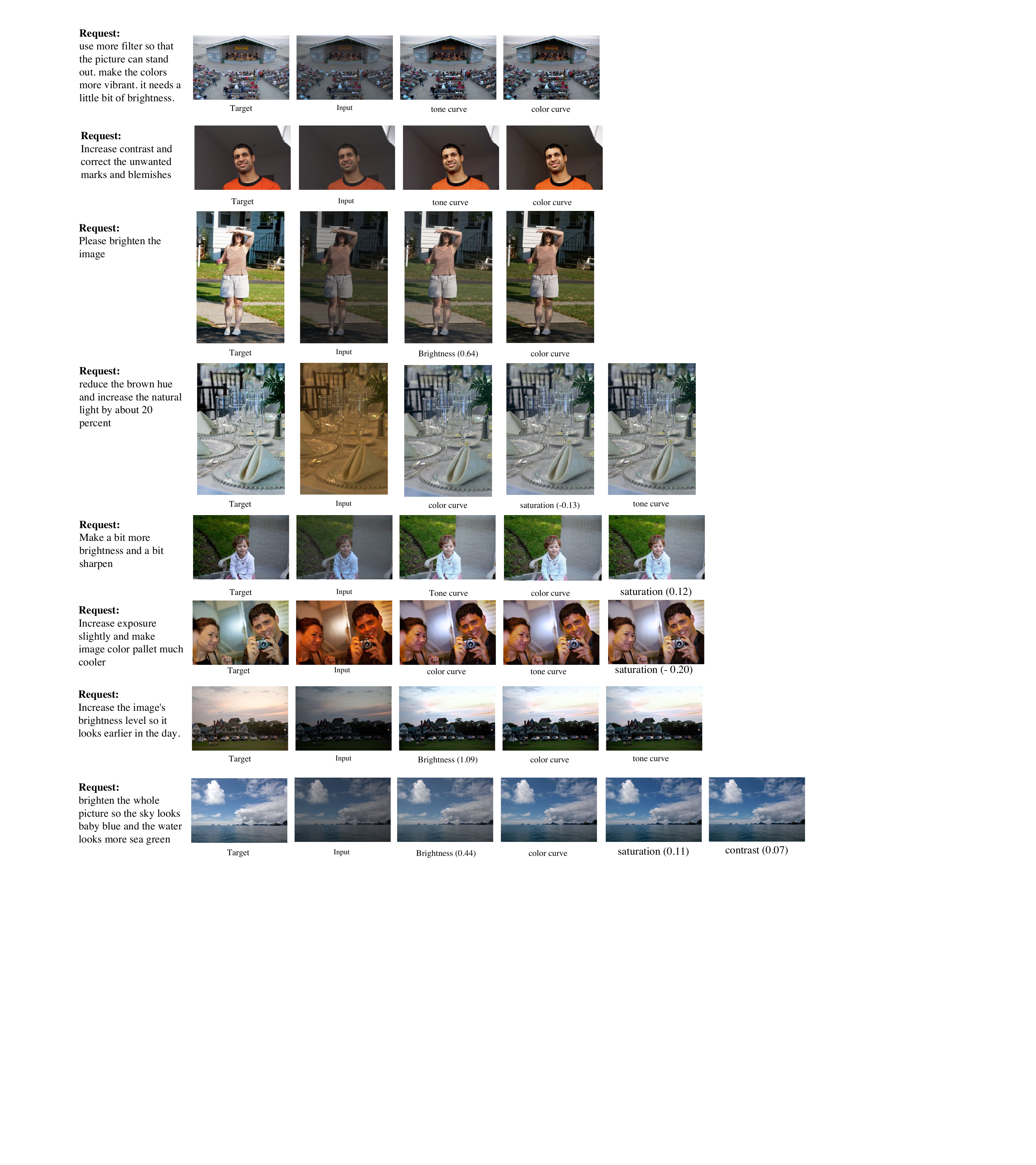}
  \caption{The visual results for T2ONet on MA5k-Req dataset.} 
  \label{fig:vis_res_FiveK}
\end{figure*}
%---------------

%---------------
\begin{figure*}[]
  \centering
  \includegraphics[width=2\columnwidth]{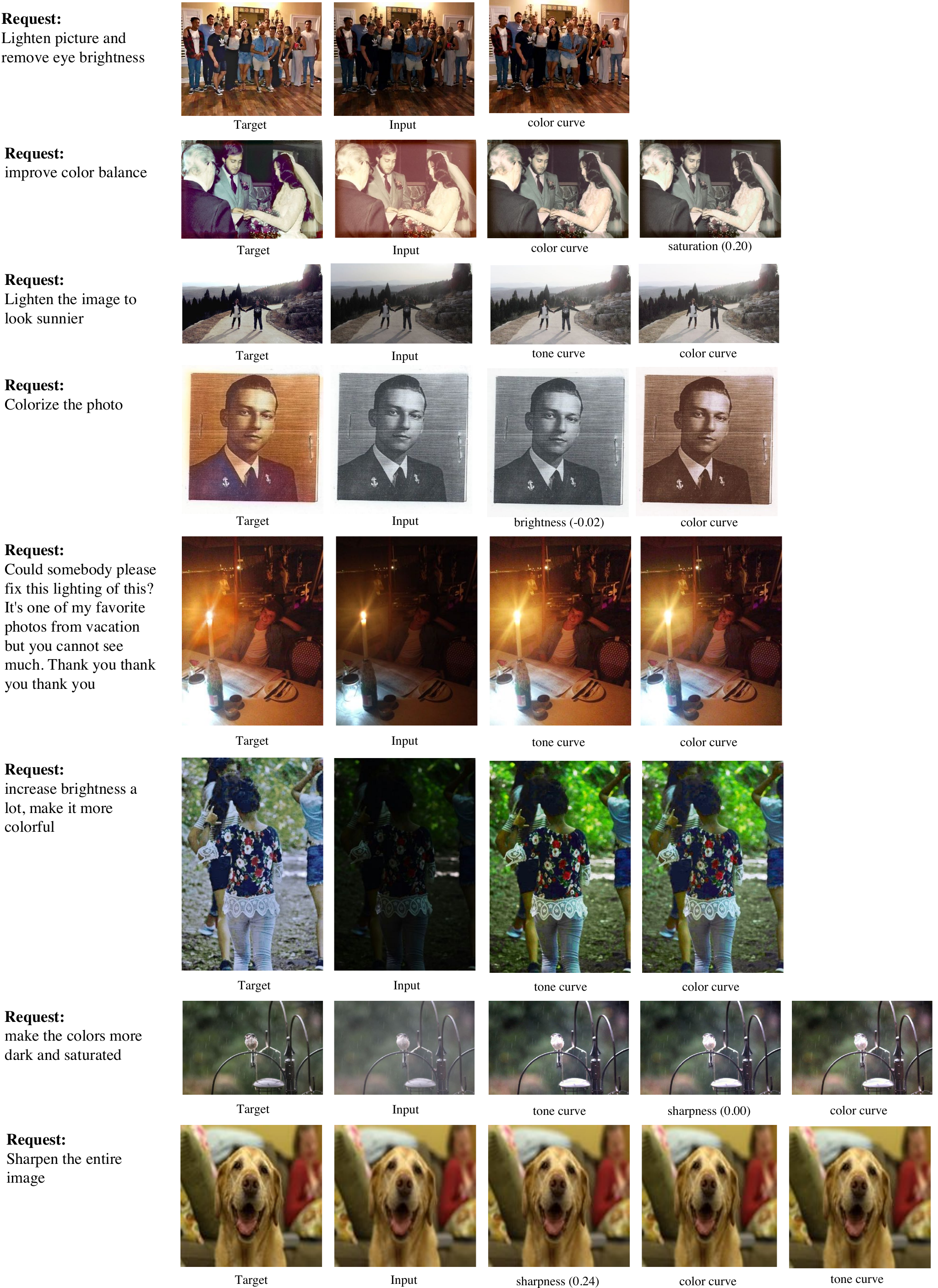}
  \caption{The visual results for T2ONet on GEIR dataset.} 
  \label{fig:vis_res_GIER}
\end{figure*}
%---------------

%---------------
\begin{figure*}[]
  \centering
  \includegraphics[width=2\columnwidth]{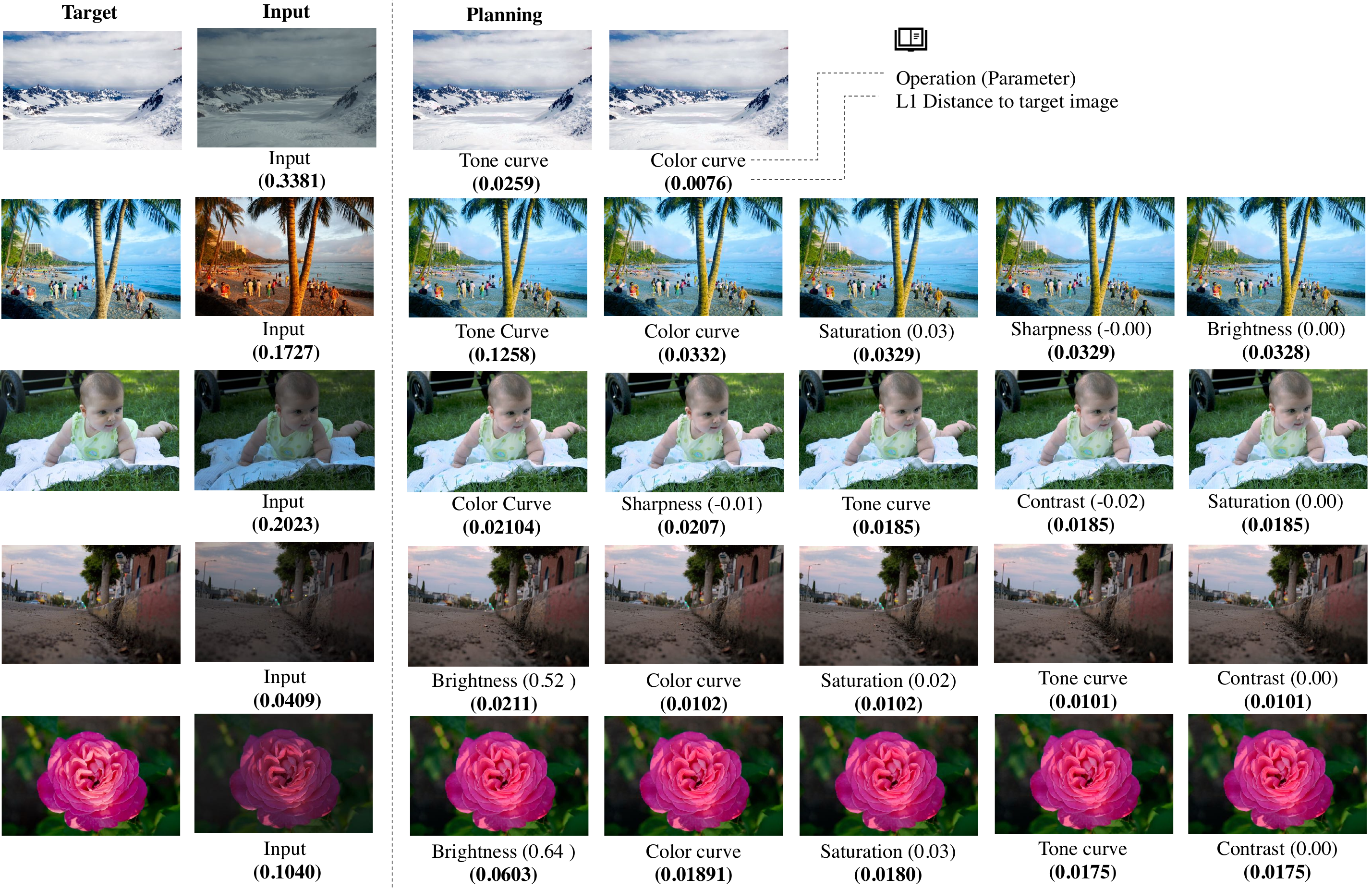}
  \caption{The visual results for operation planning.} 
  \label{fig:vis_plan_more}
\end{figure*}
%---------------

\section{More Experiment Implementation Details}\label{appx:imple_detail}
Training images are resized to $128\times 128$, and test/val images are resized to short edge 600 with aspect ratio unchanged.
The pixel value is normalized to 0-1

For T2ONet, ResNet18~\cite{he2016deep} is used to encoding image into a 512-d feature.
The word and operation embedding is 300-d, and the word embedding is initialized by GloVe.
Two-layer bi-LSTM with feature with hidden size 256 is used to encode the language request, and two-layer LSTM decoder has hidden size 512. All the other FC layers output with a 512-d feature.

For operation planning, we adopt Nelder-Mead~\cite{nelder1965simplex} for parameter optimization.
And, for lanugage-guided image editing, the training is alternatively in two losses.
For odd iteration, we only optimize $\mathcal{L}_o$ and $\mathcal{L}_\alpha$ in a teacher forcing fashion. 
For even iteration, we only optimize $\mathcal{L}_{L1}$ using the previously generated action and image as the input for the next state.
We take the top-1 operation with its parameters every step.
The final image-level $\mathcal{L}_{L1}$ can backward propagate gradients to the weights of T2OCell other than the weights of the FC layer for prediction of the operation $o$.
Hence, in all ablation study of the T2ONet, we always need the loss of $\mathcal{L}_o$ to supervise the operation selection.
The model is trained on a single GPU with a 64 batch size.

\section{Operation Implementation Details} \label{appx:op_detail}
We adopt six operations: $\mathtt{brightness}$, $\mathtt{saturation}$, $\mathtt{contrast}$, $\mathtt{sharpness}$, $\mathtt{tone}$, and $\mathtt{color}$. The operation modular network is composed of these operations in a fixed order if they are needed. With the input image $I$, parameter $p$, and output image $I'$,
the implementation of operation submodules are illustrated as follows.
% -------------------------------------
% Brightness and Saturation
\subsection{Brightness and Saturation}
The hue, saturation, value in the HSV space of image $I$ are denoted as $H(I)$, $S(I)$, $V(I)$.
Here $p$ is an unbounded scalar. Let $V'(I) = \text{clip}((1 + p)\cdot V(I), 0, 1)$ and $S'(I) = \text{clip}((1 + p)\cdot S(I), 0, 1)$, the output image for brightness operation is 
\begin{equation}
    I' = \text{HSVtoRGB}(H(I), S(I), V'(I)),
\end{equation}
and the output image for saturation operation is 
\begin{equation}
    I' = \text{HSVtoRGB}(H(I), S'(I), V(I)).
\end{equation}
The $\text{HSVtoRGB}$ is a differentiable function mapping the RGB space to HSV space, implemented via Kornia~\cite{riba2020kornia}, and $\text{clip}(x, 0, 1)$ is a clip function to clip $x$ within $0$ to $1$.
% -------------------------------------
% Contrast
\subsection{Contrast}
Contrast operation is controlled by a scalar parameter $p$, implemented following~\cite{hu2018exposure}.
First compute the luminance of image $I$ as 
\begin{equation}
    \text{Lum}(I) = 0.27 I_r + 0.67 I_g + 0.06 I_b,
\end{equation}
where $I_r$, $I_g$, $I_b$ are the RGB channels of $I$. 
The enhanced luminance is 
\begin{equation}
    \text{EnhancedLum}(I) = \frac{1}{2}(1 - \cos(\pi \cdot \text{Lum}(I))),
\end{equation}
and the image with enhanced contrast is 
\begin{equation}
    \text{EnhancedC}(I) = I\cdot \frac{\text{EnhancedLum}(I)}{\text{lum}(I)}.
\end{equation}
The output image $I'$ is the combination of the enhanced contrast and original image
\begin{equation}
    I' = (1-p)\cdot I + p\cdot \text{EnhancedC} (I).
\end{equation}

% -------------------------------------
% Sharpness
\subsection{Sharpness}
The sharpness operation is implemented by adding to the image with its second-order spatial gradient~\cite{gonzales2002digital}, expressed as 
\begin{equation}
    I' = I + p\Delta^2I,
\end{equation}
where $p$ is a scalar parameter and $(\Delta^2\cdot)$ is the Laplace operator over the spatial domain of the image. The Laplace operator is applied to each channel of the image.

%--------------------------------------
% Tint
\subsection{Tone and Color}
The tone and color operation follows curve representation~\cite{hu2018exposure}. The curve is estimated as piece-wise linear functions with $N$ pieces. The parameter $p=\{p_i\}_{i=0}^{M-1}$ is a vector of length $M$. With the input pixel $x\in [0, 1]$, the output pixel intensity is 
\begin{equation}
    f(x) = \frac{1}{Z}\sum_{i=0}^{N-1}\text{clip}(N x-i, 0, 1)p_i,
\end{equation}
where $Z=\sum_{i=0}^{N-1}p_i$. 
For tone operation, $N=M=8$, the same $f$ will apply to each of the RGB channels of the image $I$.
For color operation, three different $f$ are applied individually to each of RGB channels. Each $f(x)$ has $N=8$, leading to $M=3N=24$.

%------------------
\begin{figure*}[!tp] 
  \centering\includegraphics[width=2\columnwidth]{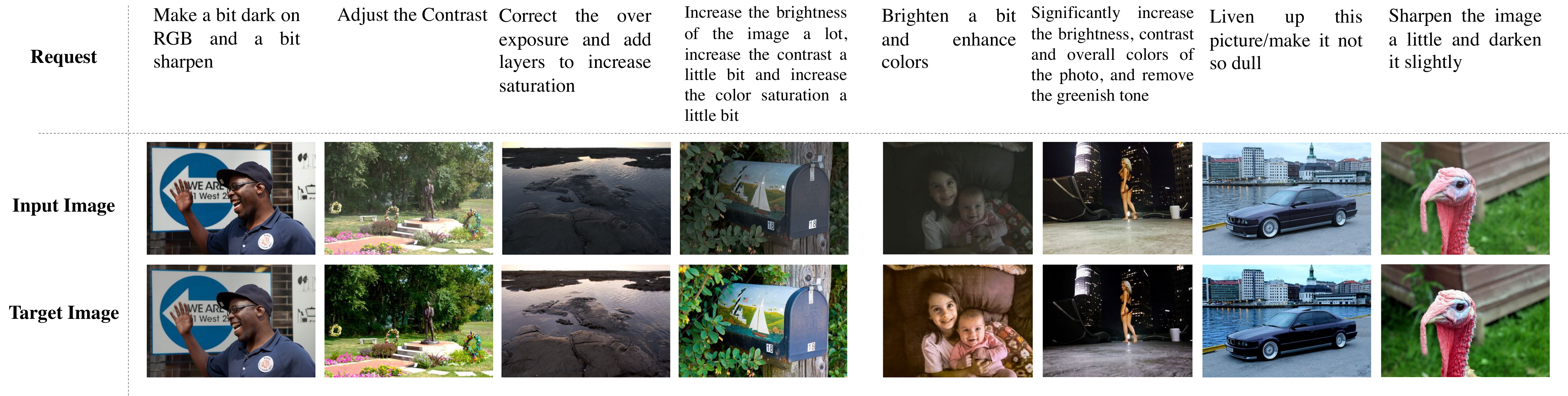} 
  \caption{Data examples draw from MA5k-Req (left) and GIER (right).} 
  \label{fig:dataset}
\end{figure*} 
%------------------

%------- Figure: user study interface ----------
\begin{figure*}[!tp] 
  \centering\includegraphics[width=2\columnwidth]{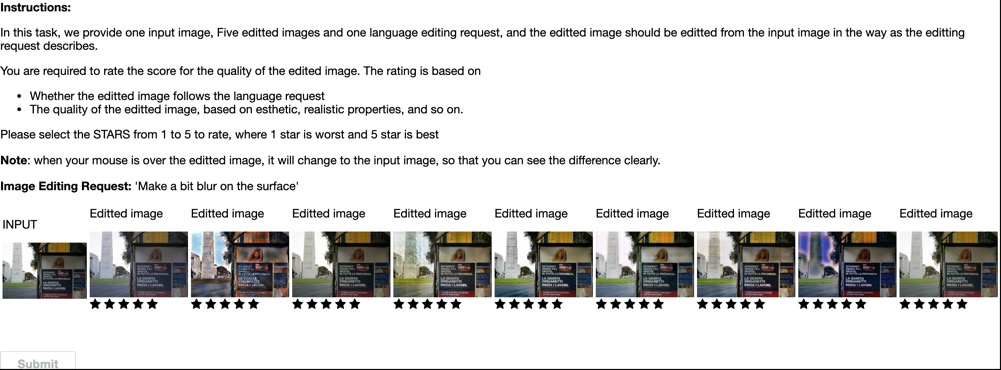} 
  \caption{The interface for user study. The edited result of all the methods are shown in random order. The worker should select the star under each edit to indicate the score they rate.} 
  \label{fig:interface}
\end{figure*}

\section{Languages for Image Variance Evaluation}\label{appx:lang}
The 10 different requests are as follows:
\begin{enumerate}[noitemsep, topsep=0pt]
\item Decrease the brightness.
\item Increase the brightness.
\item Enhance the color.
\item Decrease the color.
\item Improve contrast.
\item Reduce contrast.
\item Increase saturation.
\item Reduce saturation.
\item Increase the brightness a little.
\item Increase the brightness a lot.
\end{enumerate}

\section{Dataset}
\subsection{More Detail of MA5k-Req Collection Process}\label{appx:data_collect}
In this section we We show the worker the input and target images, and let workers write the editing request. We deploy the annotation collection interface on Amazon Mechanic Turk involving totally 268 workers for FiveK. Each request annotation is $0.03$, and we have the approvals for crowdsourcing.

We show the worker the input and target images, and let workers write the editing request. We deploy the annotation collection interface on Amazon Mechanic Turk involving totally 268 workers for FiveK and 197 workers for web images. Each request annotation is \$0.03.

For quality control, we initially collect the language requests for a subset of image pairs, and manually select good workers depending on the annotation quality. Then we only allow good workers to annotate the full dataset.
\subsection{Visualization of Dataset Samples}
Some samples draw from MA5k-Req and GIER is shown in Fig.~\ref{fig:dataset}

\section{User study details}
The interface of the user study is shown in Fig.~\ref{fig:interface}.

\newpage

{\small
\bibliographystyle{ieee_fullname}
\bibliography{egbib}
}

\end{document}